\documentclass{article}
\usepackage{graphicx} 
\usepackage{natbib}
\usepackage{amsmath}
\usepackage{amssymb}
\usepackage{amsthm}
\usepackage{xcolor}
\usepackage{hyperref}
\usepackage{mathtools}
\usepackage[preprint]{neurips_2024}
\usepackage{subfigure}
\usepackage{wrapfig,lipsum,booktabs}
\usepackage{caption}
\usepackage{subcaption}

\newtheorem{theorem}{Theorem}
\newtheorem{remark}{Remark}
\newtheorem{proposition}{Proposition}
\newtheorem{assumption}{Assumption}
\usepackage{wrapfig}
\usepackage{enumitem}
\allowdisplaybreaks
\title{Theoretical Understanding of In-Context Learning in Shallow Transformers with Unstructured Data}

\author{%
    Yue Xing\\
    Department of Statistics and Probability\\
    Michigan State University\\
    \texttt{xingyue1@msu.edu}\\
    \And
    Xiaofeng Lin\\
    Department of Statistics\\
    University of California Los Angeles\\
    \texttt{bernardo1998@g.ucla.edu}\\
    \And
    Chenheng Xu\\
    Department of Statistics\\
    University of California Los Angeles\\
    \texttt{chenhengx0101@g.ucla.edu}\\
    \And
    Namjoon Suh\\
    Department of Statistics\\
    University of California Los Angeles\\
    \texttt{namjsuh@ucla.edu}\\
    \And
    Qifan Song\\
    Department of Statistics\\
    Purdue University\\
    \texttt{qfsong@purdue.edu}
    \And
    Guang Cheng\\
    Department of Statistics\\
    University of California Los Angeles\\
    \texttt{guangcheng@ucla.edu}\\
}

\begin{document}

\maketitle

\begin{abstract}

{Large language models (LLMs) are powerful models that can learn concepts at the inference stage via in-context learning (ICL). While theoretical studies, e.g., \cite{zhang2023trained}, attempt to explain the mechanism of ICL, they assume the input $x_i$ and the output $y_i$ of each demonstration example are in the same token (i.e., structured data). However, in real practice, the examples are usually text input, and all words, regardless of their logic relationship, are stored in different tokens (i.e., unstructured data \cite{wibisono2023role}). To understand how LLMs learn from the unstructured data in ICL, this paper studies the role of each component in the transformer architecture and provides a theoretical understanding to explain the success of the architecture. In particular, we consider a simple transformer with one/two attention layers and linear regression tasks for the ICL prediction. We observe that (1) a transformer with two layers of (self-)attentions with a look-ahead attention mask can learn from the prompt in the unstructured data, and (2) positional encoding can match the $x_i$ and $y_i$ tokens to achieve a better ICL performance.}

\end{abstract}

\section{Introduction}


Large Language Models (LLMs), built upon transformer architectures, have experienced a surge in popularity in recent years, demonstrating their superior power in answering human questions and finishing various tasks. In addition to the extensive empirical studies, theoretical works started to investigate the behavior of LLMs recently. While LLMs share similar properties as other deep neural networks, LLMs still enjoy some unique characteristics, e.g., in-context learning (ICL).

Numerous studies in the literature investigate transformers to explain the success of ICL. For example, \cite{zhang2023trained} explores the convergence of transformers and their ICL ability using transformers with one linear attention layer and structured data. Structured data means that the input and output of one in-context example are in the same token in the input sequence, i.e.,
\begin{eqnarray}
    E_1=\begin{bmatrix}
        x_1 & x_2 & x_3 & \ldots & x_D & x_q\\
        y_1 & y_2 & y_3 & \ldots & y_D & 0
    \end{bmatrix},
\end{eqnarray}
where $(x_i,y_i)$s are in-context examples, and $x_q$ is the target input. After feeding $E_1$ into a well-trained transformer, the transformer can compare the similarity between $x_q$ and other $x_i$s and output a prediction of $y_q$ via assembling $y_i$s, i.e., conduct an in-context learning directly in the \textbf{inference stage} without learning the relationship between $x_q$ and $y_q$ in the \textbf{training stage}. In later works, this theoretical framework is extended to softmax attention and multi-head attention, e.g., \citet{zhang2023trained,huang2023context,cui2024superiority,chen2024training}.



However, there are some missing pieces in the study of \citet{zhang2023trained,huang2023context,cui2024superiority,chen2024training}:
\paragraph{Separating columns of $x_i$ and $y_i$} 

In real practice, when inputting the examples, e.g., sentences, each word is a token, and the whole sentence will take up several columns in the input embedding matrix. In contrast, the prompt format $E_1$ merges every information of one example into one column, which is not realistic. Observing this gap between $E_1$ and the common practice, we step one way from $E_1$.
Similar to the implementation of \citet{garg2022can} and the discussion of \citet{wibisono2023role}, we aim to train a transformer to conduct ICL using the following prompt format with ``unstructured'' data:
    $$\label{eqn:position}
    E_2=\begin{bmatrix}
    x_1 & 0 & x_2 & 0 & \ldots & x_D & 0 & x_{\text{q}} \\
    0 & y_1 & 0 & y_2 & \ldots & 0 & y_D & 0 
\end{bmatrix}.
$$
In \citet{garg2022can,wibisono2023role}, it is empirically observed that deep transformers are able to learn from unstructured data. In contrast, we explicitly show the critical difference between one and two attention layers, and provide mathematical intuitions that the attention mask together with two attention layers plays an important role.

\begin{wrapfigure}{r}{0.45\textwidth}
\vspace{-0.25in}
    \centering
        \includegraphics[scale=0.55]{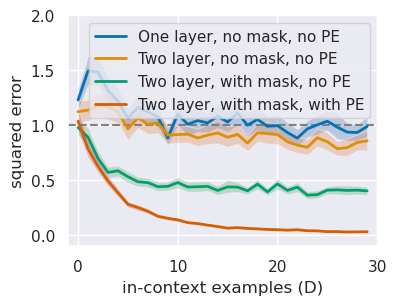}
    \caption{ICL performance with different number of layers, mask, position. The two-layer structure and the attention mask are crucial, while positional encoding significantly improves performance.}
    \label{fig:all}
    \vspace{-0.2in}
\end{wrapfigure}
\paragraph{Other components of transformers} While \citet{zhang2023trained,huang2023context,cui2024superiority,chen2024training} studies the performance of a single-layer transformer, 
they do not consider other important components: attention mask and positional encoding (PE).


While existing literature mainly focuses on large-scale experiments to demonstrate the effectiveness of LLMs in learning from structured/unstructured data, in our paper, observing the above missing analysis, we consider small-scale shallow transformers and control the experiment settings to explicitly show the benefit and interaction of each component in the transformer. In particular, our key observations are summarized in Figure \ref{fig:all}. The contributions are as follows:

\begin{itemize}[leftmargin=0.5cm]
    \item As in Figure \ref{fig:all}, for the unstructured data $E_2$, a one-layer transformer fails. With a two-layer transformer and a look-ahead attention mask, the transformer can conduct ICL. Both the two-layer structure and attention mask are crucial. (Section \ref{sec:two_layer})
    
    \item With a two-layer transformer and attention mask, PE can further improve performance. Based on our theoretical results, PE can better match $x_i$ and $y_i$s in the first attention layer to transform the unstructured data to structured data. A larger input embedding dimension and multiple layers of attention can also help improve the matching. Then, in the second attention layer, one can use structured data to perform ICL. (Section \ref{sec:pos})

    
\end{itemize}

Note that our study doesn't claim the superiority of unstructured data over structured data. Based on our understanding, the transformer needs to match $(x_i,y_i)$s, i.e., process the unstructured data into structured data in order to perform ICL. From this perspective, structured data is actually preferred to unstructured data. However, the study of unstructured data is still necessary because it is not avoidable in real practice. For instance, text prompts are usually not curated and contain no direct structural/causal information between words. The LLMs need to capture the relationship among the words first, and then perform predictions.\color{black}

\section{Related Works}
\paragraph{Theoretical studies} It is well known that LLMs can perform in-context learning: the LLMs can infer the output based on the input and given examples in the prompt in the testing stage, without being trained using related data in the training stage. Various empirical studies are conducted to understand ICL, e.g., \cite{garg2022can}. 

Recently, there are some researches studying the theoretical perspectives of ICL. The first direction considers the simple one-attention-layer architecture to study how linear regression and other simple statistics tasks can be performed in such a neural network using structured data. Various theories can be investigated in terms of the expressive power of the transformer, training convergence, and generalization, e.g., \cite{von2023transformers,ahn2023transformers,akyurek2022learning,zhang2023trained,ahn2023linear,zhang2023trained,huang2023context,han2023context,cui2024superiority,chen2024training,kim2024transformers,nichani2024transformers}. In addition to softmax attention, the linear attention is one of the proposed methods that has been widely studied in the theoretical literature as well, e.g., ~\cite{li2020linear, katharopoulos2020transformers, shen2021efficient, zheng2022linear, liu2023linrec, qin2022cosformer}. 

\paragraph{Transformer Architecture}

 As will be introduced later, there are several components in the transformer architecture considered in our paper: attention mask, multiple layers of attention, PE.

Masking is needed to prevent the attention mechanism of a transformer from ``cheating'' in the decoder when training.
While tasks such as translation need the information from the whole sentence, in other tasks, attention mask has shown empirical successes in many studies, specifically in the field of computer vision and multi-modality models~\cite{cheng2022masked, fan2021mask, pang2019mask, harley2017segmentation, song2020bi}.

Another important component in transformer architecture is its compositional structure of multiple attention layers.  
We are aware of two works~\cite{van2019does, simoulin2021many} which provided interesting analysis on the internal layer mechanisms of BERT.

Positional encoding (PE), suggested by~\cite{vaswani2017attention}, successfully facilitates the transformers to know the positional information of each word in the sentence. 
PE originally introduced for language-based models has many variants based on different data modalities; 
the instances include tree-based data~\cite{shiv2019novel}, graph data~\cite{bruel2022rewiring}, time series data~\cite{nyborg2022generalized}, video data~\cite{ning2020polar}, etc.
Nonetheless, there is not much literature on the theoretical understanding of PE.
In the context of ICL,~\cite{bai2023transformers} proves transformers with PE can implement more complex ICL procedures involving in-context algorithm selection.

\section{Notations and Architecture}\label{sec:notation}
In the following, we define the data distribution and the transformer to be considered.

\paragraph{Data} In this study, we assume all examples and the final query have an $x$ following i.i.d. $N(0,I_d)$. The response $y$ is a noiseless linear model, $y=x^{\top}\theta$, for some $\theta\sim N(0,I_d/d)$. All examples and the query share the same $\theta$ in the same prompt. For different prompts, $\theta$s are different. This data assumption is widely used in existing theoretical literature \cite{zhang2023trained,huang2023context,cui2024superiority}. Since our focus on the role of PE rather than the relationship between data distribution of ICL ability, we do not consider the other complex relationship in \cite{chen2024training}.

To train the transformer, we assume there are infinitely many samples of prompts with their examples and minimize the following loss
\begin{eqnarray*}
    \mathbb{E}_{\{x_i\}_{i\in[D]},x_{\text{q}},\theta}\sum_{i=1}^D (y_i-\widehat y_i)^2+(y_{\text{q}}-\widehat y_{\text{q}})^2.
\end{eqnarray*}
where $\widehat y_i$ and $\widehat y_{\text{q}}$ are the prediction of $y_i$ and $y_{\text{q}}$, whose definition can be found in the later paragrph.

\paragraph{Transformer architecture}  When considering a transformer with two layers of attention, we consider the following structure. In the first attention layer, we use a softmax (self-)attention with a look-ahead attention mask. In particular, define $O$ as the output of the first attention layer, we have
$$f(E):=O:=W_{in}E+P+W_{V,1}(W_{in}E+P)\phi\left( (W_{in}E+P)^{\top} W_{KQ,1}(W_{in}E+P)+M\right),$$
where $P\in\mathbb{R}^{p\times (2D+1)}$ is the PE matrix, $E\in\mathbb{E}^{(d+1)\times (2D+1)}$ is the input matrix, $W_{in}\in\mathbb{R}^{p\times (d+1)}$ is an embedding matrix mapping $(x_i,0)$ and $(0,y_i)$ columns of $E$ into a $p$-dimensional space. The function $\phi$ is a column-wise softmax operator, i.e., for a vector $v$, $\phi(v)=\exp(v)/\|\exp(v)\|_1$. In addition, the attention mask matrix $M$ is an upper-triangular matrix, with all non-zero elements as $-\infty$. In the second layer, we further pass $O$ into another attention layer to obtain the corresponding output $O_2$, and finally output $W_{out}(O_2)_{:,2D+1}$, a linear function of the last column of $O_2$, as the predicted $y_q$. The corresponding parameters in the second layer are $W_{KQ,2}$ and $W_{V,2}$. 
In the second layer, we use linear attention to simplify the derivations, and its detailed architecture design can be found in the corresponding theorem. The setup of one-layer transformers is similar, with details postponed to the corresponding theorem.

In the attention mechanism, the quantities calculated from $\phi$ are called the attention scores. When $\phi$ is the column-wise softmax operator, for each column of $$\phi\left( (W_{in}E+P)^{\top} W_{KQ,1}(W_{in}E+P)+M\right),$$all the elements are non-negative, and sum up to 1. As a result, when multiplying the $W_{V,1}(W_{in}E+P)$, it can be viewed as a weighted average of the different columns, i.e., $W_{V,1}(W_{in}E+P)\phi_{:,i}=\sum \phi_{j,i}(W_{V,1}(W_{in}E+P))_{j}$, which can be viewed as a kernel method. Similarly, for linear attention, e.g., \cite{zhang2023trained}, we obtain a linear kernel.


\color{black}

\paragraph{Experiment Settings}

Throughout this paper, we modify the implementation of \cite{garg2022can} to examine the effect of the different components in the transformer. In the tasks, we take $x$ dimension as $d=5$ and the number of examples in each prompt as $D=30$. We train different models for 50k to ensure the convergence of optimization and the ICL prediction performance no longer changes. A bad performance under such a large number of steps implies the failure of convergence, e.g., Figure \ref{fig:two_layer}. Table \ref{tab:code} in the appendix summarizes how to control each component in the transformer. In addition, for all experiments, we show the result for one repetition in the main content, and postpone the ten-repetition results in Appendix \ref{sec:appendix:simulation}.

In order to correctly measure the effect of PE and mask, we use the following procedure to calculate the testing loss. In the first step, to predict $y_1$ using $x_1$, we only include $(x_1,0)$ in the prompt. To predict $y_2$ using $(x_1,0),(0,y_1)$ and $(x_2,0)$, we take $E_2=((x_1,0),(0,y_1),(x_2,0))$. With an attention mask, the training and testing loss is reduced to the same formula. Without an attention mask, during training, PE will ``steal" information from the later $y_i$ tokens and achieve an almost-zero loss. However, such a strategy does not help prediction in the inference phase when the input $E_2$ does not contain the corresponding $y_i$s.


\color{black}

\section{Two-Layer Transformer + Attention Mask}\label{sec:two_layer}
In this section, we demonstrate that the two-layer structure and the attention mask together play a crucial role in ICL when the examples are in a format of $E_2$. We do not consider PE in this section.

\subsection{Empirical Observation}

Under $E_2$, it is crucial that the transformer architecture is capable of connecting each $y_i$ with its corresponding $x_i$. This goal cannot be achieved without any positional information built into the model. 
We find that the transformer can connect $x_i$s to $y_i$s when there is more than one layer of attention, along with the attention mask. 

\begin{figure}
\centering
\includegraphics[scale=0.6]{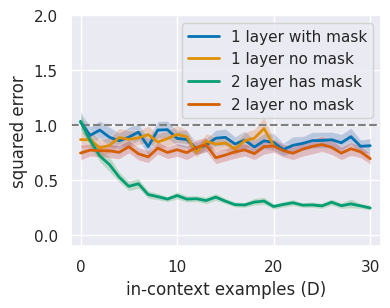}
   \caption{Performance of one-layer and two-layer transformers with/without attention mask (no PE). }
   \vspace{-0.2in}
    \label{fig:two_layer}
\end{figure}

Figure \ref{fig:two_layer} demonstrates our preliminary experiment results. We take $x_i$ dimension $d=5$, embedding dimension $p=64$, and number of heads $h=1$. We use $E_2$ as the input to the transformer and examine the ICL performance. Figure \ref{fig:two_layer} considers transformers of one/two layers of attention, with/without attention mask, and without PE. One can see that when using a two-layer transformer together with the attention mask, the prediction loss quickly drops when increasing the number of examples $D$, while missing either the second layer or the attention mask results in a failure of ICL. 


\subsection{Why One-Layer Transformer Fails}

In the following, we explain the poor performance of single-layer transformers.

Intuitively, $\widehat y_{\text{q}}$ is a weighted average of the last entry across all tokens in the prompt. In a simplified scenario where $W^V_{p,:}=(0,\ldots,0 , 1)$, this becomes $\widehat y_{\text{q}}=\sum_{i=1}^D 0\cdot\phi_{2i-1,D}+\sum_{i=1}^Dy_i\cdot\phi_{2i,D}$, where $\phi_{k,D}$ is the attention score of $x_{\text{q}}$ with the $k$-th token in the prompt. 
One attention layer, regardless of $h$, aims to capture the column dependency of its input prompt. Therefore, given a single layer of attention,  (i) $\phi_{2i,D}$'s can only capture the dependency information between $(0,y_i)$ tokens and $(x_{\text{q}},0)$ and fails to deliver any useful in-context knowledge to $\widehat y_{\text{q}}$; (ii) $\phi_{2i-1,D}$'s do capture similarity information about $x_i$'s and $x_{\text{q}}$ but cannot pass this information to $\widehat y_{\text{q}}$ as they are multiplied by 0's. This leads to the failure of $\widehat y_{\text{q}}$ prediction.

The following theorem converts the above intuition into a theorem:
\begin{theorem}[One layer attention is not sufficient]\label{thm:one_layer}
   Consider a transformer with one layer of softmax attention. Assume there is no $W_{in}$ and PE. Assume $W_{KQ,1}$ is in the form of
    \begin{eqnarray*}
        W_{KQ,1}=\begin{bmatrix}
            A & 0\\ b & 0
        \end{bmatrix}
    \end{eqnarray*}
    for some $A\in\mathbb{R}^{d\times d}$ and $b\in\mathbb{R}^d$ such that $\|A\|$ and $\|b\|$ are both bounded. Assume $D\rightarrow\infty$, then regardless of whether the attention mask is in the transformer or not, the optimal solution of $b$ satisfies $\|b\|^2\rightarrow 0$, and the minimal value of $\mathbb{E}(\widehat y_q-y_q)^2$ is in $\Theta(1)$.
\end{theorem}
The proof of Theorem \ref{thm:one_layer} can be found in Appendix \ref{sec:appendix:one_layer}. The key idea to prove Theorem \ref{thm:one_layer} is that, when taking $\theta$ and $-\theta$, the corresponding $\widehat y_q$s are similar. Thus, the optimal solution is that $\widehat y_q\approx0$.


\color{black}

\subsection{Why Two-Layer Transformer Succeeds}

To explain why two-layer transformer + attention mask facilitates ICL in $E_2$, the following theorem demonstrates that even without PE, the transformer is able to perform ICL using unstructured data:

\begin{wrapfigure}{r}{0.5\textwidth}
\centering
    \vspace{-0.2in}
    \includegraphics[scale=0.6]{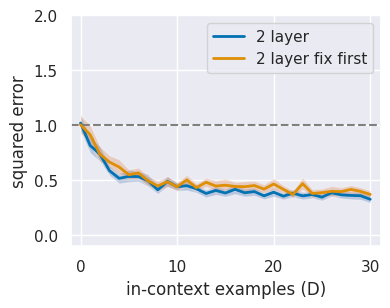}\vspace{-0.1in}
    \caption{ICL prediction performance of training two layers vs training the second layer only.}\label{fig:lazy}\vspace{-0.2in}
\end{wrapfigure}

\begin{theorem}[Two layers with attention mask facilitate ICL]\label{thm:two_layer_lazy}
   Assume there is no $W_{in}$. Assume $W_{KQ,1}$ is a zero matrix, and $W_{V,1}=I_{d+1}$. With the attention mask, the output $f(E_2)$ of the first layer is
        \begin{eqnarray}\label{eqn:first_layer_output}
            O=\begin{bmatrix}
                2x_1 & \frac{1}{2}x_1 & \frac{1}{3}x_1+\frac{4}{3}x_2 & \frac{1}{4}x_1+\frac{1}{4}x_2 & \ldots \\
                0 & \frac{3}{2}y_1 & \frac{1}{3}y_1 & \frac{1}{4}y_1+\frac{5}{4}y_2 & \ldots 
            \end{bmatrix}.
        \end{eqnarray}
        In the second layer, considering a linear attention without mask and taking 
        \begin{eqnarray*}
            \phi\left(E^{\top}(W^K)^{\top}W^QE\right)_{:,2D+1}=\frac{1}{\log(D)}E^{\top}\begin{bmatrix}
            I_d & 0\\0 & 0
        \end{bmatrix}\begin{bmatrix}
            x_q\\0
        \end{bmatrix},\label{eqn:linear}
        \end{eqnarray*}
        we have
        \begin{eqnarray*}
            \widehat{y}_{\text{q}} \approx\frac{x_{\text{q}}^{\top}}{\log(D)}\left(\sum_{i} \left( \frac{1}{i}-\frac{1}{2D} \right) x_ix_i^{\top}\right)\theta +\frac{\epsilon}{\log(D)}\label{eqn:weight}\approx y_{\text{q}}+\frac{\epsilon}{\log(D)},
        \end{eqnarray*}
        with $\epsilon/\log(D)=O_p(\sqrt{d/\log(D)})$,
          i.e., the second layer is able to predict $y_{\text{q}}$. The notation $\epsilon$ represents some cross terms of $x_{\text{q}}^{\top}x_ix_j^{\top}\theta$, which is negligible when $D$ is large enough. 
\end{theorem}

The proof of Theorem \ref{thm:two_layer_lazy} can be found in Section \ref{sec:appendix:formula} in Appendix.


\color{black}

To further justify Theorem \ref{thm:two_layer_lazy}, we also conduct the corresponding simulations. We fix the first layer $W_{KQ,1}$ to be zero, and train the other parameters in the transformer. From Figure \ref{fig:lazy}, there is only a slight performance difference regarding whether $W_{KQ,1}$ is fixed or not. It implies that the training of the first layer focuses on mixing all columns of $E_2$ together, and as long as the columns of $E_2$ are mixed in an asymmetric way among different $i$s, one can perform ICL in the second layer.


\section{Positional Encoding}\label{sec:pos}

This section studies the effect of the completely trainable PE. In short, as in Figure \ref{fig:all}, given infinite training prompts, the ICL performance is improved by PE.

\subsection{Intuition: Why Positional Encoding Helps}\label{sec:pos_help}
\color{black}

To explain why PE improves the ICL performance, we draw the heat map of the attention scores of the two attention layers without/with PE (Figure \ref{fig:attn_no_pos} and Figure \ref{fig:attn_pos}). 

\begin{figure}[h]
\centering
\begin{minipage}{0.42\textwidth}
\centering
\includegraphics[scale=0.29]{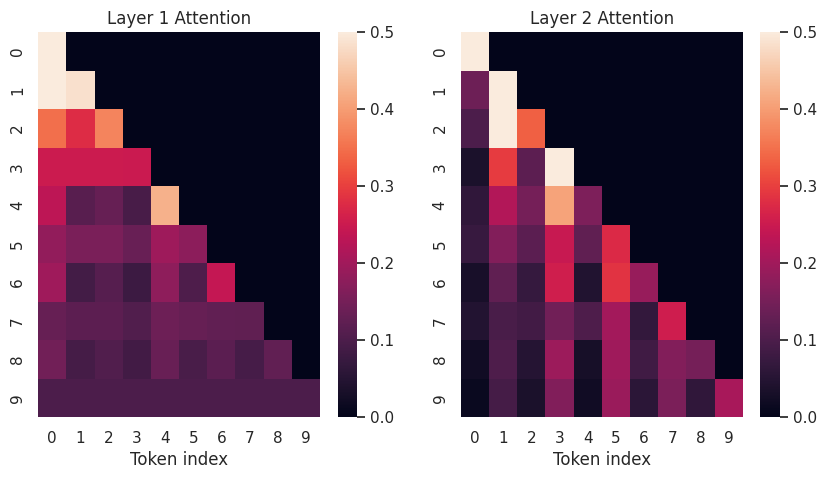}
        \caption{Attention scores of the first 5 input pairs on a single head, two layers, with mask, $E_2$ format. One prompt. Each row is the attention of one token. No PE.} \label{fig:attn_no_pos}
\end{minipage}\hspace{0.3in}
\begin{minipage}{0.42\textwidth}
\centering
\includegraphics[scale=0.29]{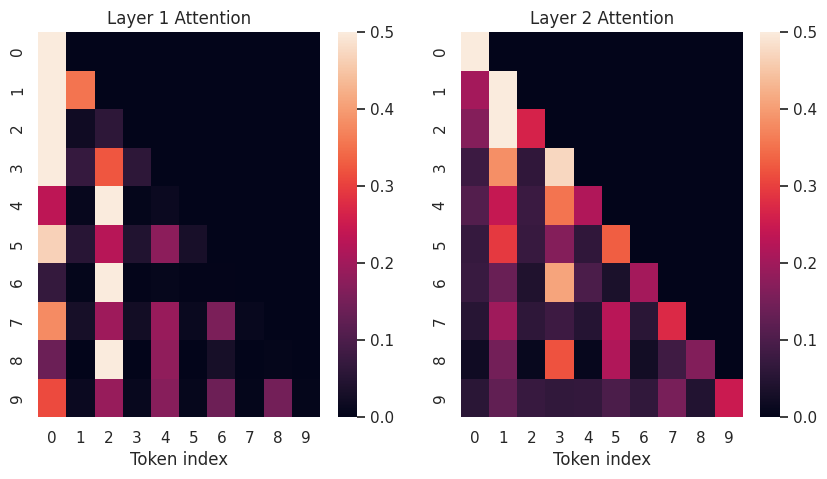}
        \caption{Attention scores of the first 5 input pairs on a single head, two layers, with mask, $E_2$ format. One prompt. Each row is the attention of one token. With PE.} \label{fig:attn_pos}
\end{minipage}

\end{figure}

For the sake of the readability of the figures, we only plot attention scores for the first 5 pairs (i.e., 10 columns) in Figure \ref{fig:attn_no_pos} and Figure \ref{fig:attn_pos}. The full plots can be found in Figures \ref{fig:attn_no_pos_full} and \ref{fig:attn_pos_full} in the appendix.

There are two observations from Figure \ref{fig:attn_no_pos} and Figure \ref{fig:attn_pos}. First, Figure \ref{fig:attn_no_pos} verifies the analysis in Section \ref{sec:two_layer}. In the attention score matrix of the first layer (left panel of Figure \ref{fig:attn_no_pos}), the attention scores are almost the same along each row. On the other hand, in the second layer, in each row, the high scores often appear on the odd positions (position index starts from 0 in the figures), which contains information of $y_i$s. These observations suggest that the first layer aggregates different columns, and the second layer learns from the examples and makes an ICL prediction. 

Second, in Figure \ref{fig:attn_pos}, when adding PE, the attention scores in the first layer are almost zero for the odd positions in each row and apply a high weight for $x_i$s. Denote the $P$ matrix as
    $$P=\begin{bmatrix}
        p_{11} & p_{12} & \ldots & p_{D2} & p_{q1}
    \end{bmatrix}=\begin{bmatrix}
        p_{11}^x & p_{12}^x & p_{21}^x & p_{22}^x & p_{31}^x & \ldots & p_{D1}^x & p_{D2}^x & p_{q1}^x\\
        p_{11}^y & p_{12}^y & p_{21}^y & p_{22}^y & p_{31}^y & \ldots & p_{D1}^y & p_{D2}^y & p_{q1}^y
    \end{bmatrix},$$
where $p_{ij}^x\in\mathbb{R}^{d}$ and $p_{ij}^y\in\mathbb{R}$ when $P\in \mathbb{R}^{(d+1)\times D}$ (no $W_{in}$).  
Recall that in the formulation, assuming $W_{V,1}=I$ and no $W_{in}$, the {(2$i$th column of) the }output of the first attention layer becomes
\begin{eqnarray*}
    f(E_2)=\begin{bmatrix}
        \ldots & \sum_{j\leq i} \phi_{2j-1,2i}(x_j+p_{j1}^x)+\phi_{2j,2i}p_{j2}^x +p_{i2}^x& \ldots\\
        \ldots & \sum_{j\leq i}\phi_{2j-1,2i}p_{j1}^y+\phi_{2j,2i}(y_j+p_{j2}^y)  + y_{i}+p_{i2}^y & \ldots
    \end{bmatrix}.
\end{eqnarray*}
Introducing PE helps the attention score matrix focus more on matching $x_i$ and $y_i$, instead of naively taking an average to obtain a mixture of different $(x_j,y_j)$s. {As illustrated by the previous toy example and will be theoretically investigated later}, with PE, for each $(0,y_i)$ column, the corresponding attention score of $(x_i,0)$ ($\phi_{2i-i,2i}$) and $(0,y_i)$ ($\phi_{2i,2i}$), will be {large} (around 1/2), while the other attention scores ($\phi_{2j-1,2i}$ and $\phi_{2j,2i}$ for $j<i$) will be almost zero. Therefore, $2i$th column of $f(E_2)$ reduces to a weighted average of $(x_i,0)$ and $(0,y_i)$ {(plus some position encoding constants)}.

\subsection{Positional Encoding without Input Embedding}\label{sec:benefit}
\color{black}

When studying the expressive power of PE, since all features in $x$ are symmetric, we employ the following assumption in the transformer parameters to simplify the analysis:
\begin{assumption}\label{assumption:nn}
    In the first layer, we assume $W_{V,1}=v I_{d+1}$, and $$W_{KQ,1}=\begin{bmatrix}
    v_xI_d & 0\\
    0 & 0
\end{bmatrix}.$$
\end{assumption}
While Theorem \ref{thm:two_layer_lazy} considers a scenario where $W_{KQ,1}$ is a zero matrix, Assumption \ref{assumption:nn} allows us to adjust the attention score based on $x_i$s and PE. This assumption also aligns with the format of the optimal solution in \cite{zhang2023trained,cui2024superiority,huang2023context}.

To quantify the benefit of positional encoding, while \cite{garg2022can} considers a completely trainable PE, to simplify the theoretical analysis, we construct the following PE with a trainable scalar $c$:
\begin{eqnarray}\label{eqn:new_pe}
c\begin{bmatrix}
        1 & 1    & 0 & \ldots & 0\\
        0 & 0&  1 &  1  \ldots & 0\\
        \ldots
    \end{bmatrix}.
\end{eqnarray}
In the above construction, the PE is the same for the two columns of the same example, and is perpendicular to the other examples. 

\begin{wraptable}{r}{0.5\linewidth}
        \centering
    \begin{tabular}{c|c}
        PE method & MSE \\\hline
        New PE in (\ref{eqn:new_pe}) & 0.1992 \\
        No PE & 0.4849 \\
        Completely trainable PE & 0.1877
    \end{tabular}
    \caption{MSE using different PE methods. }
    \label{tab:new_pe}
\end{wraptable}

To validate the effectiveness of the above PE, we conduct a small simulation. We use the implementation of \cite{garg2022can} in the simulation, the transformer has 2 layers, 1 attention head, and is trained for 500k iterations. In each prompt, $d=5$ and $D=30$. To demonstrate the clear difference among the no PE case, new PE, and the completely trainable PE, we take $p=64$. One can see that the overall MSE of using (\ref{eqn:new_pe}) is close to the completely trainable PE, and can effectively improve the performance compared to the no PE case.\color{black}

Using the above PE construction as in (\ref{eqn:new_pe}), the following theorem demonstrates that the first attention layer with PE successfully matches $x_i$s with their corresponding $y_i$s. We consider the scenario without the input embedding matrix $W_{in}$ in this section and consider $W_{in}$ in the next section.
 \begin{theorem}[First layer without embedding matrix]\label{thm:icl_no_emb}
    When $D<d$, using the PE construction in (\ref{eqn:new_pe}), taking proper $c$ and $v_x$ such that $v_x=\Theta(\sqrt{\log(d)/d})$ and $c=c_0\sqrt{d\log(d)}$ for some large constant $c_0$, with $d\rightarrow\infty$, with probability tending to 1 over the randomness of the examples $(x_i,y_i)$s, the output of the first attention layer combines $x_i$ and $y_i$ well uniformly, i.e.,
    \begin{eqnarray*}
        O_{:,2i} = \frac{v}{2}\begin{bmatrix}
    x_i+p_{i1}^x+p_{i2}^x\\y_i
    \end{bmatrix}+\begin{bmatrix}
            p_{i2}^x\\
            y_i
        \end{bmatrix}+o_p.
    \end{eqnarray*}
    In addition,
    \begin{eqnarray*}
        O_{:,2i-1}=(v+1)\begin{bmatrix}
            x_i+p_{i1}^x\\0
        \end{bmatrix}+o_p.
    \end{eqnarray*}
The notation $o_p$ represents a negligible term with probability tending to 1 uniformly for all $i$s. The index number starts from 1.
\end{theorem}

Theorem \ref{thm:icl_no_emb} shows that when designing $P$ properly, the attention mechanism helps match $x_i$s and their corresponding $y_i$s in the first layer. To prove Theorem \ref{thm:icl_no_emb}, the key is to use probability bounds to quantify the behavior of the attention scores. The details are in Section \ref{sec:appendix:proof} in Appendix.

While the above theorem shows how $x_i$s and $y_i$s are matched in the $2i$th column in the first layer, there are two issues. First, the asymptotic of $d\rightarrow\infty$, i.e., infinite number of features, is impractical. Second, the constraint on $D<d$ also limits the ICL performance. 

If we do not apply input embedding $W_{in}$ and insist on increasing $D$, when $D\gg d$, the any PE does not have enough flexibility to match $x_i$s and $y_i$s. Since $D$ is much larger than $p$, some columns of $P$ should be in a similar direction. As a result, PE cannot perfectly match $x_i$ and $y_i$ in such a direction. The detailed mathematical description can be found as follows:

\begin{theorem}[PE is not flexible when $D\gg p=d+1$]\label{prop:large_D}
  The output of the first layer satisfies
   \begin{eqnarray*}
    O_{:,2i}&=&E\phi\left( (E+P)^{\top}W_{KQ,1}(E+P)+M  \right)_{:,2i}+\begin{bmatrix}
        p_{i2}^x\\y_i
    \end{bmatrix}\\
    &=& \frac{1}{B_i}\sum_{j\leq i}  \begin{bmatrix}
        x_j+p_{j1}^x\\
        0
    \end{bmatrix} \exp\left( v_x ( x_j+p_{j1}^x )^{\top}p_{i2}^x\right)+\frac{1}{B_i}\sum_{j<i} \begin{bmatrix}
        p_{j2}^x\\
        y_j
    \end{bmatrix}\exp(v_x(p_{j2}^x)^{\top}p_{i2}^x)\\
    &&+\underbrace{\frac{1}{B_i}\begin{bmatrix}
        p_{i2}^x\\y_i
    \end{bmatrix}\exp(v_x\|p_{i2}\|^2)}_{\text{We want to match $x_i$ with this term.}}
\end{eqnarray*}
where$$B_i=\sum_{j<i}\left[\exp(v_x x_j^{\top}p_{i2}^x)+\exp(v_x (p_{j2}^x)^{\top}p_{i1}^x )\right]+\exp(v_xx_i^{\top}p_{i2}^x+v_x c^2)+\exp(v_x c^2).$$
 When $D$ is large enough, for large $i$, under some general assumptions on PE,
 \begin{eqnarray}\label{eqn:large_D}
     \underbrace{\left\|\sum_{j< i}  \begin{bmatrix}
        x_j+p_{j1}^x\\
        0
    \end{bmatrix} \exp\left( v_x ( x_j+p_{j1}^x )^{\top}p_{i2}^x\right)\right\|_2}_{\text{Terms that are not related to $x_i$.}}\gg\underbrace{\left\| \begin{bmatrix}
        x_i+p_{i1}^x\\0
    \end{bmatrix}\exp\left( v_x(x_i+p_{i1}^x)^{\top}p_{i2}^x \right) \right\|_2}_{\text{We want to match this term with $y_i$.}}.
 \end{eqnarray}
 
\end{theorem}

Based on (\ref{eqn:large_D}) in Theorem \ref{prop:large_D}, when matching $x_i$ and $y_i$ in the first layer, $O_{:,2i}$ not only contains information about $x_i$, but also the other columns for other $x_j$s with $j<i$. This implies that for large $D$, $x_i$ and $y_i$ cannot be perfectly matched in the first layer. The proof and the condition in PE can be found in Appendix \ref{sec:appendix:proof}.

\subsection{Benefit of Positional Encoding with Large Input Embedding Dimension}\label{sec:benefit_pe}
\color{black}

When introducing a large embedding dimension, a large $p$ improves the flexibility of the PE design, allowing us to have more in-context examples.

\begin{theorem}[First layer with embedding matrix]\label{prop:icl_emb}
    Under some conditions of $W_{in}$ and $W_{KQ,1}$, when $D<p$, and $p\gg d$, using the PE in (\ref{eqn:new_pe}), taking $v_x=\Theta(\sqrt{\log p/p})$, $c=c_0\sqrt{p\log p}$ for some large enough constant $c_0$, the output of the first attention layer combines $x_i$ and $y_i$ well, i.e., for $i\rightarrow\infty$,
\begin{eqnarray*}
    O_{:,2i-1}=(v+1)\left(W_{in}\begin{bmatrix}
        x_i\\
        0
    \end{bmatrix}+p_{i1}\right)+o_p.
\end{eqnarray*}
In terms of $O_{:,2i}$, similarly,
\begin{eqnarray*}
    O_{:,2i}=\frac{v}{2}\left(W_{in}\begin{bmatrix}
        x_i\\y_i
    \end{bmatrix} +p_{i1}+p_{i2}\right)+\left(W_{in}\begin{bmatrix}
        0\\y_i
    \end{bmatrix} +p_{i2}\right)+o_p.
\end{eqnarray*}
\end{theorem}

From Theorem \ref{prop:icl_emb}, the output of the first layer can be viewed as structured data with PE as bias. The conditions on $W_{in}$ and $W_{KQ,1}$ and the proof of Theorem \ref{prop:icl_emb} are postponed to the appendix. In the second layer, following \cite{zhang2023trained} for linear attention, one can obtain a good ICL performance:
\begin{theorem}[ICL performance with embedding matrix]\label{thm:icl_prediction}
    Consider the following input for a transformer with one single-head attention layer:
\begin{eqnarray*}
    E_{:,2i-1}=(v+1)\left(W_{in}\begin{bmatrix}
        x_i\\
        0
    \end{bmatrix}+p_{i}\right),
\end{eqnarray*}
and
\begin{eqnarray*}
    E_{:,2i}=\frac{v}{2}W_{in}\begin{bmatrix}
        x_i\\0
    \end{bmatrix} +\left( \frac{v}{2}+1 \right)W_{in}\begin{bmatrix}
        0\\y_i
    \end{bmatrix} +(v+1)p_{i},
\end{eqnarray*}
where $p_i=p_{i1}=p_{i2}$ from the designed PE (\ref{eqn:new_pe}).

Then,
there exists some transformer with linear attention such that $\mathbb{E}(\widehat y_q-y_q)^2=O(d/D)$.

\end{theorem}

\begin{remark}
    To simplify our analysis, we use linear attention in the second attention layer to avoid the bias from PE. In the first layer, while successfully matching $x_i$ and $y_i$,  PE can greatly affect the attention scores because of the nonlinear $\exp$ function in the softmax operator. However, in the second layer, PE has no significant effect on the attention score due to the use of linear attention. If the second layer is also softmax attention, PE can lead to a bias in the final prediction. In such a case, the corresponding $c$ in the PE formulation (\ref{eqn:new_pe}) should be smaller to balance the trade-off between the matching effect in the first layer and the ICL prediction performance. 

\end{remark}

Due to the page limit, we postpone some additional simulation studies to Section \ref{sec:appendix:simulation} in Appendix. Briefly speaking, we keep increasing the number of examples $D$ and observe that the ICL performance using small $p$ cannot be as good as the case with large $p$ but small $D$.
\color{black}

\subsection{With PE but without Attention Mask}
\color{black}

While in the previous sections we explain the importance of attention mask and two attention layers, in this section, we consider the case where no attention mask is applied. Briefly speaking, when PE is applied, even for a one-layer transformer, the minimal training loss is close to zero through stealing the label information from later tokens. On the other hand, in terms of the testing performance, the loss does not converge.

\begin{proposition}\label{thm:one_layer_with_pe}
    For transformers with one softmax attention layer, following the same condition as Theorem \ref{prop:icl_emb}, when taking $c=c_0\sqrt{p\log p}$,
\begin{eqnarray*}
    O_{:,2i-1}=\frac{v}{2}\left(W_{in}\begin{bmatrix}
        x_i\\y_i
    \end{bmatrix} +p_{i1}+p_{i2}\right)+\left(W_{in}\begin{bmatrix}
        x_i\\0
    \end{bmatrix} +p_{i1}\right)+o_p.
\end{eqnarray*}
On the other hand, following the same argument as in Theorem \ref{thm:one_layer}, when predicting $y_i$, we remove the column $(0,y_i)$ and later columns, then the best possible prediction loss becomes $\Theta(1)$.
\end{proposition}

\begin{wrapfigure}{r}{0.45\textwidth}
\vspace{-0.2in}
    \centering
        \includegraphics[scale=0.6]{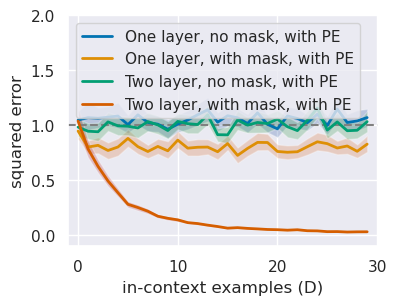}
    \caption{ICL performance of one/two attention layer, with/without mask, with PE.}
    \vspace{-0.5in}
    \label{fig:pe}
\end{wrapfigure}

From Proposition \ref{thm:one_layer_with_pe}, the information of $y_i$ at the $2i$th token is leaked to the $2i-1$th token, and therefore the training loss can be very small.

We conduct some simulations to examine the ICL prediction performance with PE when changing the number of layers and the attention mask. In Figure \ref{fig:pe}, with PE, one can see that the ICL prediction loss is small only when there are two layers and the attention mask is applied. Besides Figure \ref{fig:pe}, we also observe that the training loss for the one-/two-layer transformers with PE without mask are almost zero, i.e., the information of $y_i$ at the $2i$th token is completely leaked to the $2i-1$th token.

\section{Conclusions}
\color{black}
In this paper, we conduct experiments and explain why transformers can learn the unstructured data $E_2$. Different from the structured prompt format $E_1$ where $x$ and $y$ are in the same column, $E_2$ separates them. When training with infinite prompts, (1) the transformer should be of at least two attention layers with the look-ahead attention mask to learn $E_2$; (2) PE can further connect $x$ and $y$ and balance the mean and variance in the prediction loss, and in addition  (3) a high embedding dimension $p$ improves the ICL performance.

There are several possible future directions. First, as mentioned in Theorem \ref{thm:icl_prediction}, linear attention in the second layer facilitates ICL using the unstructured data even if PE is amplified in the first layer with softmax attention. Further analysis can be conducted to quantify how softmax attention in the second layer affects the ICL performance. Second, one may also consider the effect of multiple attention layers. In particular, when there are more than two layers, it is not necessary to match the two layers perfectly. One can apply a weaker PE and accumulate its effect through the layers. In addition, as ICL compares the similarity of $x_q$ and other $x_i$s to determine the weight for $y_i$s, this process may also happen in all the attention layers. Detailed quantification of these two effects can deepen the understanding of how multiple attention layers work. Third, another possible direction is to study the effect of PE in the training process. This paper mainly considers PE when training the transformer, with almost infinite examples, and does not consider the generalization gap. One may study the generalization performance given PE. On one hand, for the completely trainable PE, the number of trainable parameters can be significantly increased, and the generalization gap can be larger. On the other hand, from how the PE is applied in the transformer, its effect on the generalization gap may differ from the other parameters. Finally, one may extend the analysis to chain-of-thought, e.g., \cite{li2024chain,li2023dissecting}.

In terms of broader impact, since this paper primarily focuses on theoretical understanding of the mechanism of transformer in learning unstructured data rather than developing new methodologies, there is no direct negative social impact. For limitations, since this paper only considers a simple transformer architecture, it may not capture the full characteristics as the transformers used in real practice. Future works may be considered to generalize the analysis to large-scale transformers.





\newpage
\onecolumn
\bibliographystyle{plainnat}
\bibliography{ref}

\newpage
\appendix

Below is a summary of the appendix:
\begin{itemize}
    \item Section \ref{sec:appendix:one_layer}: proofs for one-layer transformers.
    \item Section \ref{sec:appendix:formula}: proofs for two-layer transformers without PE.
    \item Section \ref{sec:appendix:proof}: proofs for for transformers with PE.
    \item Section \ref{sec:appendix:configuration}: details of configurations of experiments.
    \item Section \ref{sec:appendix:other}: additional figures.
    \item Section \ref{sec:appendix:simulation}: additional simulations.
\end{itemize}

Notation-wise, in the proofs, we use ``$o$" to represent some negligible terms which are statistical estimation error such as $(\sum_{i=1}^n x_i)/n=\mathbb{E}(\sum_{i=1}^n x_i)/n+o$, and use ``$\approx$" to represent approximation error caused by operators such as $\sum_{i=1}^n 1/i\approx \log(n)$.



\section{One-Layer Transformer without PE}\label{sec:appendix:one_layer}
\begin{proof}[Proof of Theorem \ref{thm:one_layer}]
    We first consider the case with an attention mask.
    
    The key idea of the proof is that, since $\theta\sim N(0,I_d/d)$, we also have $-\theta\sim N(0,I_d/d)$. When minimizing the loss associated with $y=\theta^{\top}x$ and $y=-\theta^{\top}x$, the optimal solution is that $\widehat y\approx 0$.

    When $y=\theta^{\top}x$, we have
    \begin{eqnarray*}
        \widehat y_q&=&\frac{\sum 0\cdot\exp(x_i^{\top}Ax_q)+\sum y_i\exp(y_ib^{\top}x_q)}{\sum \exp(x_i^{\top}Ax_q)+\exp(x_q^{\top}Ax_q) + \sum \exp(y_ib^{\top}x_q) }\\
        &=&\frac{\sum y_i\exp(y_ib^{\top}x_q)}{\sum \exp(x_i^{\top}Ax_q)+\exp(x_q^{\top}Ax_q) + \sum \exp(y_ib^{\top}x_q) }.
    \end{eqnarray*}
    When $D\rightarrow\infty$, fixing $\theta$, we have
    \begin{eqnarray*}
        \widehat y_q&=&\frac{\sum y_i\exp(y_ib^{\top}x_q)}{\mathbb{E}[\sum \exp(x_i^{\top}Ax_q)+\exp(x_q^{\top}Ax_q) + \sum \exp(y_ib^{\top}x_q) ]}\\
        &&+\frac{\sum y_i\exp(y_ib^{\top}x_q)}{\sum \exp(x_i^{\top}Ax_q)+\exp(x_q^{\top}Ax_q) + \sum \exp(y_ib^{\top}x_q) }\\
        &&-\frac{\sum y_i\exp(y_ib^{\top}x_q)}{\mathbb{E}[\sum \exp(x_i^{\top}Ax_q)+\exp(x_q^{\top}Ax_q) + \sum \exp(y_ib^{\top}x_q) ]}\\
        &=& \frac{\sum y_i\exp(y_ib^{\top}x_q)}{\mathbb{E}[\sum \exp(x_i^{\top}Ax_q)+\exp(x_q^{\top}Ax_q) + \sum \exp(y_ib^{\top}x_q) ]} + o\\
        &=&\frac{\mathbb{E}[\sum y_i\exp(y_ib^{\top}x_q)]}{\mathbb{E}[\sum \exp(x_i^{\top}Ax_q)+\exp(x_q^{\top}Ax_q) + \sum \exp(y_ib^{\top}x_q) ]}\\
        &&+\frac{\sum y_i\exp(y_ib^{\top}x_q)}{\mathbb{E}[\sum \exp(x_i^{\top}Ax_q)+\exp(x_q^{\top}Ax_q) + \sum \exp(y_ib^{\top}x_q) ]}\\
        &&-\frac{\mathbb{E}[\sum y_i\exp(y_ib^{\top}x_q)]}{\mathbb{E}[\sum \exp(x_i^{\top}Ax_q)+\exp(x_q^{\top}Ax_q) + \sum \exp(y_ib^{\top}x_q) ]}+o\\
        &=&\frac{\mathbb{E}[\sum y_i\exp(y_ib^{\top}x_q)]}{\mathbb{E}[\sum \exp(x_i^{\top}Ax_q)+\exp(x_q^{\top}Ax_q) + \sum \exp(y_ib^{\top}x_q) ]}+o\\
        &:=&\frac{\mathbb{E}[\sum y_i\exp(y_ib^{\top}x_q)]}{\mathcal{E}(\theta)}+o.
    \end{eqnarray*}
    The term $\mathcal{E}(\theta)$ is the generic representation for the term on the denominator. The term "$o$" in the above derivation represents some negligible term. In particular, since
    \begin{eqnarray*}
        &&{\sum \exp(x_i^{\top}Ax_q)+\exp(x_q^{\top}Ax_q) + \sum \exp(y_ib^{\top}x_q) }\\
        &=&{\mathbb{E}[\sum \exp(x_i^{\top}Ax_q)+\exp(x_q^{\top}Ax_q) + \sum \exp(y_ib^{\top}x_q) ]}+O_p(\sqrt{D}),
    \end{eqnarray*}
    using Taylor expansion, we obtain that
    \begin{eqnarray*}
        &&\frac{1}{\sum \exp(x_i^{\top}Ax_q)+\exp(x_q^{\top}Ax_q) + \sum \exp(y_ib^{\top}x_q) }\\
        &&-\frac{1}{\mathbb{E}[\sum \exp(x_i^{\top}Ax_q)+\exp(x_q^{\top}Ax_q) + \sum \exp(y_ib^{\top}x_q) ]}\\
        &=&\frac{\sum \exp(x_i^{\top}Ax_q)+\exp(x_q^{\top}Ax_q) + \sum \exp(y_ib^{\top}x_q)-\mathbb{E}[\sum \exp(x_i^{\top}Ax_q)+\exp(x_q^{\top}Ax_q) + \sum \exp(y_ib^{\top}x_q) ]}{\mathbb{E}^2[\sum \exp(x_i^{\top}Ax_q)+\exp(x_q^{\top}Ax_q) + \sum \exp(y_ib^{\top}x_q) ]}\\
        &&+o\left(\frac{\sum \exp(x_i^{\top}Ax_q)+\exp(x_q^{\top}Ax_q) + \sum \exp(y_ib^{\top}x_q)-\mathbb{E}[\sum \exp(x_i^{\top}Ax_q)+\exp(x_q^{\top}Ax_q) + \sum \exp(y_ib^{\top}x_q) ]}{\mathbb{E}^2[\sum \exp(x_i^{\top}Ax_q)+\exp(x_q^{\top}Ax_q) + \sum \exp(y_ib^{\top}x_q) ]}\right)\\
        &=&\frac{O_p(\sqrt{D})}{\mathbb{E}^2[\sum \exp(x_i^{\top}Ax_q)+\exp(x_q^{\top}Ax_q) + \sum \exp(y_ib^{\top}x_q) ]}\\
        &=&\frac{O_p(1/\sqrt{D})}{\mathbb{E}[\sum \exp(x_i^{\top}Ax_q)+\exp(x_q^{\top}Ax_q) + \sum \exp(y_ib^{\top}x_q) ]},
    \end{eqnarray*}
    and it is negligible compared to
    $$ \frac{1}{\mathbb{E}[\sum \exp(x_i^{\top}Ax_q)+\exp(x_q^{\top}Ax_q) + \sum \exp(y_ib^{\top}x_q) ]}.$$

   On the other hand, for the numerator, we have
    \begin{eqnarray*}
        \mathbb{E}\left[\sum y_i\exp(y_ib^{\top}x_q)\right]
        &=& D\mathbb{E}[y_i\exp(y_ib^{\top}x_q)]\\
        &=&D\mathbb{E}[x_i^{\top}\theta\exp(\theta^{\top}x_i b^{\top}x_q)]\\
        &=&D \|\theta\|^2b^{\top}x_q\exp\left(\frac{1}{2}\|\theta\|^2b^{\top}x_q\right).
    \end{eqnarray*}
    Similarly, when taking $y=-\theta^{\top}x$, we have $\mathcal{E}(\theta)=\mathcal{E}(-\theta)$. In addition, we again obtain that
    \begin{eqnarray*}
        \mathbb{E}\left[\sum y_i\exp(y_ib^{\top}x_q)\right]
        =D \|\theta\|^2b^{\top}x_q\exp\left(\frac{1}{2}\|\theta\|^2b^{\top}x_q\right).
    \end{eqnarray*}
    When calculating the prediction loss, we have
    \begin{eqnarray*}
        &&\mathbb{E}_{\theta}(\widehat y_q-y_q)^2\\
        &=&\frac{1}{2}\mathbb{E}_{\theta}(\widehat y_q-y_q)^2+\frac{1}{2}\mathbb{E}_{-\theta}(\widehat y_q-y_q)^2\\
        &=&\mathbb{E}_{\theta}\left[\frac{1}{2}\left(\frac{D\|\theta\|^2b^{\top}x_q\exp(\|\theta\|^2b^{\top}x_q/2)}{\mathcal{E}(\theta)}-\theta^{\top}x_q\right)^2+\frac{1}{2}\left(\frac{D\|\theta\|^2b^{\top}x_q\exp(\|\theta\|^2b^{\top}x_q/2)}{\mathcal{E}(\theta)}+\theta^{\top}x_q\right)^2\right]+o,
    \end{eqnarray*}
    to minimize which we need to take $\|b\|\rightarrow 0$.

    Besides $x_q$, the above argument also applies for $x_i$s for large $i$s.

    To analyze the case without attention mask, the different thing compared to the above steps is that we know the later columns in the input, i.e., when $y=\theta^{\top}x$,
        \begin{eqnarray*}
        \widehat y_j&=&\frac{\sum 0\cdot\exp(x_i^{\top}Ax_j)+\sum y_i\exp(y_ib^{\top}x_j)+ y_j\exp(y_j b^{\top}x_j)}{\sum \exp(x_i^{\top}Ax_j)+\exp(x_j^{\top}Ax_j) + \sum \exp(y_ib^{\top}x_j) + \exp(y_j b^{\top}x_j) }\\
        &=&\frac{D\|\theta\|^2b^{\top}x_j\exp(\|\theta\|^2b^{\top}x_j/2) + x_j^{\top}\theta\exp(b^{\top}x_jx_j^{\top}\theta)}{\mathcal{E}(\theta)}+o.
    \end{eqnarray*}
    When $y=-\theta^{\top}x$,
        \begin{eqnarray*}
        \widehat y_j=\frac{D\|\theta\|^2b^{\top}x_j\exp(\|\theta\|^2b^{\top}x_j/2) - x_j^{\top}\theta\exp(-b^{\top}x_jx_j^{\top}\theta)}{\mathcal{E}(\theta)}+o.
    \end{eqnarray*}
    With large probability, when $\|b\|$ is finite, $D\|\theta\|^2b^{\top}x_j\exp(\|\theta\|^2b^{\top}x_j/2)\gg \exp(b^{\top}x_jx_j^{\top}\theta)$, thus the optimal $b$ still satisfies that $\|b\|\rightarrow0$.
\end{proof}

\section{Two-Layer Transformer without PE}\label{sec:appendix:formula}

\begin{proof}[Proof of Theorem \ref{thm:two_layer_lazy}]
    
In the first layer, since $W^K=W^Q=0$, we have
\begin{eqnarray*}
    \phi( (W^KE_2)^{\top}(W^QE_2) )=\phi\left(\begin{bmatrix}
        0 & 0 & \ldots & 0\\
        \ldots\\
        0 & 0 & \ldots & 0
    \end{bmatrix}\right)=\begin{bmatrix}
        1 & \frac{1}{2} & \frac{1}{3} & \ldots & \frac{1}{2D+1}\\
        0 & \frac{1}{2} & \frac{1}{3} & \ldots & \frac{1}{2D+1}\\
        0 & 0 & \frac{1}{3} & \ldots & \frac{1}{2D+1}\\
        \ldots\\
        0 & 0 & 0 & \ldots & \frac{1}{2D+1}
    \end{bmatrix}.
\end{eqnarray*}
Further, because $W^V=I_{d+1}$, the output of the first layer becomes
\begin{eqnarray}
            f(E_2)&=& E+ E\phi( (W^KE_2)^{\top}(W^QE_2) )\nonumber\\
            &=& E+\begin{bmatrix}
                x_1 & \frac{1}{2}x_1 & \frac{1}{3}x_1+\frac{1}{3}x_2 & \frac{1}{4}x_1+\frac{1}{4}x_2 & \ldots & \frac{1}{2D+1}x_{\text{q}}+\frac{1}{2D+1}\sum x_i\\
                0 & \frac{1}{2}y_1 & \frac{1}{3}y_1 & \frac{1}{4}y_1+\frac{1}{4}y_2 & \ldots & \frac{1}{2D+1}\sum y_i
            \end{bmatrix}\nonumber\\
            &=&\begin{bmatrix}
                2x_1 & \frac{1}{2}x_1 & \frac{1}{3}x_1+\frac{4}{3}x_2 & \frac{1}{4}x_1+\frac{1}{4}x_2 & \ldots & \frac{2D+2}{2D+1}x_{\text{q}}+\frac{1}{2D+1}\sum x_i\\
                0 & \frac{3}{2}y_1 & \frac{1}{3}y_1 & \frac{1}{4}y_1+\frac{5}{4}y_2 & \ldots & \frac{1}{2D+1}\sum y_i
            \end{bmatrix},\nonumber
        \end{eqnarray}
        which is the formula (\ref{eqn:first_layer_output}) in Section \ref{sec:two_layer}.
        
        To obtain (\ref{eqn:weight}), in the second layer, taking a linear $\phi$ and $(W^K)^{\top}W^Q=\begin{bmatrix}
            I_d/\log(D) & 0\\0 & 0
        \end{bmatrix}$, the last column is
        \begin{eqnarray*}
             &&\phi\left(f(E_2)^{\top}\begin{bmatrix}
            I_d/\log(D) & 0\\0 & 0
        \end{bmatrix} f(E_2) \right)_{2D+1}\\
        &=&\frac{1}{\log(D)} \begin{bmatrix}
            2x_1 & 0\\
            \frac{1}{2}x_1 &  \frac{3}{2}y_1\\
            \frac{1}{3}x_1+\frac{4}{3}x_2  &  \frac{1}{3}y_1\\
            \frac{1}{4}x_1+\frac{1}{4}x_2 & \frac{1}{4}y_1+\frac{5}{4}y_2 \\
            \ldots\\
             \frac{2D+2}{2D+1}x_{\text{q}}+\frac{1}{2D+1}\sum x_i & \frac{1}{2D+1}\sum y_i
        \end{bmatrix}\begin{bmatrix}
                 \frac{2D+2}{2D+1}x_{\text{q}}+\frac{1}{2D+1}\sum x_i\\
                \frac{1}{2D+1}\sum y_i
            \end{bmatrix}\\
            &=&\frac{1}{\log(D)} \begin{bmatrix}
            2x_1 & 0\\
            \frac{1}{2}x_1 &  \frac{3}{2}y_1\\
            \frac{1}{3}x_1+\frac{4}{3}x_2  &  \frac{1}{3}y_1\\
            \frac{1}{4}x_1+\frac{1}{4}x_2 & \frac{1}{4}y_1+\frac{5}{4}y_2 \\
            \ldots\\
             \frac{2D+2}{2D+1}x_{\text{q}}+\frac{1}{2D+1}\sum x_i & \frac{1}{2D+1}\sum y_i
        \end{bmatrix}\left(\begin{bmatrix}
            \frac{2D+2}{2D+1}x_{\text{q}}\\
            0
        \end{bmatrix}+\begin{bmatrix}
                 \frac{1}{2D+1}\sum x_i\\
                \frac{1}{2D+1}\sum y_i
            \end{bmatrix}\right),
        \end{eqnarray*}
        where
        \begin{eqnarray*}
           &&\frac{1}{\log(D)} \begin{bmatrix}
            2x_1 & 0\\
            \frac{1}{2}x_1 &  \frac{3}{2}y_1\\
            \frac{1}{3}x_1+\frac{4}{3}x_2  &  \frac{1}{3}y_1\\
            \frac{1}{4}x_1+\frac{1}{4}x_2 & \frac{1}{4}y_1+\frac{5}{4}y_2 \\
            \ldots\\
             \frac{2D+2}{2D+1}x_{\text{q}}+\frac{1}{2D+1}\sum x_i & \frac{1}{2D+1}\sum y_i
        \end{bmatrix}\begin{bmatrix}
            \frac{2D+2}{2D+1}x_{\text{q}}\\
            0
        \end{bmatrix}\\
        &=&\underbrace{\frac{1}{\log(D)}\frac{2D+2}{2D+1}\begin{bmatrix}
            2x_1^{\top}x_{\text{q}}\\
            \frac{1}{2}x_1^{\top}x_{\text{q}}\\
            (\frac{1}{3}x_1+\frac{4}{3}x_2)^{\top}x_{\text{q}}\\
            \frac{1}{4}x_1^{\top}x_{\text{q}}+\frac{1}{4}x_2^{\top}x_{\text{q}}\\
            \ldots\\
            \frac{1}{2D+2}x_{\text{q}}^{\top}x_{\text{q}}+\frac{1}{2D+1}(\sum x_i)^{\top}x_{\text{q}}
        \end{bmatrix}}_{:=A_1},
        \end{eqnarray*}
        and the other term is
        \begin{eqnarray*}
             \underbrace{\frac{1}{\log(D)} \begin{bmatrix}
            2x_1 & 0\\
            \frac{1}{2}x_1 &  \frac{3}{2}y_1\\
            \frac{1}{3}x_1+\frac{4}{3}x_2  &  \frac{1}{3}y_1\\
            \frac{1}{4}x_1+\frac{1}{4}x_2 & \frac{1}{4}y_1+\frac{5}{4}y_2 \\
            \ldots\\
             \frac{2D+2}{2D+1}x_{\text{q}}+\frac{1}{2D+1}\sum x_i & \frac{1}{2D+1}\sum y_i
        \end{bmatrix}\begin{bmatrix}
                 \frac{1}{2D+1}\sum x_i\\
                \frac{1}{2D+1}\sum y_i
            \end{bmatrix}}_{:=A_2}.
        \end{eqnarray*}

        While $A_1$ is closely related to $x_q$, $A_2$ is more related to the average of $x_i$s. 
        
        To show that $A_2$ is negligible, we have
        \begin{eqnarray*}
            A_2&=&\frac{1}{\log(D)}\begin{bmatrix}
                2x_1^{\top}\frac{1}{2D+1}\sum x_i \\
                \frac{1}{2}x_1^{\top}\frac{1}{2D+1}\sum x_i+\frac{3}{2}y_1\frac{1}{2D+1}\sum y_i\\
                \ldots\\
                \left(\frac{2D+2}{2D+1}x_{\text{q}}+\frac{1}{2D+1}\sum x_i\right)\frac{1}{2D+1}\sum x_i + \frac{1}{(2D+1)^2}(\sum y_i)^2.
            \end{bmatrix}
        \end{eqnarray*}
        Since $x_i$s are i.i.d. $N(0,I_d)$, with probability tending to 1, uniformly for all $i$s,
        \begin{eqnarray*}
            \frac{1}{2D+1}x_i^{\top}(\sum x_i)=O_p(\sqrt{d/D}).
        \end{eqnarray*}
        On the other hand, in terms of $A_1$, we have all the elements in $A_1$ are in $O_p(\sqrt{d})$. As a result, $A_2$ is negligible.

        When last row of $W_{out}W^V$ as $[0,\ldots,0,1]$, we have
        \begin{eqnarray*}
             &&\widehat{y}_{\text{q}}\\ &=& f(E_2)_{d+1,:}f(E_2)^{\top}f(E_2)_{2D+1}/\log(D)\nonumber\\
            &=&f(E_2)_{d+1,:}A_1+o\nonumber\\
            &=&\frac{1}{\log(D)}\frac{2D+2}{2D+1}\begin{bmatrix}
                0 & \frac{3}{2}y_1 & \frac{1}{3}y_1 & \frac{1}{4}y_1+\frac{5}{4}y_2 & \ldots & \frac{1}{2D+1}\sum y_i
            \end{bmatrix}\begin{bmatrix}
            2x_1^{\top}x_{\text{q}}\\
            \frac{1}{2}x_1^{\top}x_{\text{q}}\\
            (\frac{1}{3}x_1+\frac{4}{3}x_2)^{\top}x_{\text{q}}\\
            \frac{1}{4}x_1^{\top}x_{\text{q}}+\frac{1}{4}x_2^{\top}x_{\text{q}}\\
            \ldots\\
            \frac{1}{2D+2}x_{\text{q}}^{\top}x_{\text{q}}+\frac{1}{2D+1}(\sum x_i)^{\top}x_{\text{q}}
        \end{bmatrix}+o\nonumber\\
            &=&\frac{1}{\log(D)}\frac{2D+2}{2D+1}\left( \frac{3}{4}\theta^{\top}x_1x_1^{\top}x_{\text{q}} + \frac{1}{9}\theta^{\top}x_1x_1^{\top}x_{\text{q}}+\frac{1}{16}\theta^{\top}x_1x_1^{\top}x_{\text{q}} +\frac{5}{16}\theta^{\top}x_1x_1^{\top}x_{\text{q}}+\ldots\right)+\frac{\epsilon}{\log(D)}+o\\
            &=&\frac{1}{\log(D)}x_{\text{q}}^{\top}\left(\sum_{i} \left( \frac{1}{i}-\frac{1}{2D} \right) x_ix_i^{\top}\right)\theta +\frac{1}{\log(D)}\epsilon+o
        \end{eqnarray*}
        where the term $\epsilon$ is the cross term of $y_ix_j^{\top}x_{\text{q}}$. In particular,
        \begin{eqnarray*}
            \epsilon&=&\frac{2D+2}{2D+1}\sum_{i\neq j}w_{ij}y_ix_j^{\top}x_q,
        \end{eqnarray*}
        with
        \begin{eqnarray*}
            &&w_{12}=\frac{2}{3}-\frac{1}{2D+1}+o,\; w_{13}=\frac{2}{5}-\frac{1}{2D+1}+o,\; w_{1j}=\frac{2}{2j-1}-\frac{1}{2D+1}+o,\\
            &&w_{21}=\frac{2}{4}-\frac{1}{2D+1}+o,\;w_{23}=\frac{2}{6}-\frac{1}{2D+1}+o,\;w_{2j}=\frac{2}{2j}-\frac{1}{2D+1}+o,\\
            &&w_{ij}=\frac{2}{2j-2+i}-\frac{1}{2D+1}+o,
        \end{eqnarray*}
        where $o$ represents the difference between $\sum 1/i^2$ and $\int 1/i^2$.
        
        When taking the expectation of $\epsilon$, $\mathbb{E}\epsilon=0$. In terms of $\mathbb{E}\epsilon^2$, we have
        \begin{eqnarray*}
            \mathbb{E}(x_i^{\top}x_j^{\top}x_q)^2=d,
        \end{eqnarray*}
        and
        \begin{eqnarray*}
            \mathbb{E}\epsilon^2 &\leq& d\sum_{i=1}^{D}\sum_{j> i}\left(\frac{2}{2i+j-2}+\frac{2}{2j+i-2}-\frac{2}{2D+1}\right)^2\\
            &\leq& 2d\sum_{i=1}^{D}\sum_{j> i}\left(\frac{2}{2i+j-2}\right)^2+\left(\frac{2}{2j+i-2}\right)^2+o\\
            &\approx& 2d\sum_{i=1}^D\frac{6}{i}+o\\
            &\approx& 12d\log(D)+o.
        \end{eqnarray*}
        When $D$ is large enough, $\epsilon/\log(D)=O_p(\sqrt{d/\log(D)})\xrightarrow{P} 0$.
\end{proof}

\section{Two-Layer Transformer with Attention Mask and PE}\label{sec:appendix:proof}
\subsection{Theorem \ref{thm:icl_no_emb}}
\begin{proof}[Proof of Theorem \ref{thm:icl_no_emb}]
    To prove Theorem \ref{thm:icl_no_emb}, since we assume $c=c_0\sqrt{d\log(d)}$, the positional encoding will be the dominate component when determining the attention score matrix.

We calculate $O_1$ step by step. For $(E+P)^{\top}W_{KQ,1}(E+P)$, we have
\begin{eqnarray*}
    E^{\top}W_{KQ,1}E=\frac{p}{d}\begin{bmatrix}
        v_x \|x_1\|^2 & 0 & v_x x_1^{\top}x_2 & 0 & v_x x_1^{\top}x_3 & \ldots & v_x x_1^{\top}x_q\\
        0 & v_y y_1^2 & 0 & v_y y_1y_2 & 0 & \ldots & 0\\
        v_x x_2^{\top}x_1 & 0 & v_x \|x_2\|^2 & 0 & v_x x_2^{\top}x_3 & \ldots & v_x x_2^{\top}x_q\\
        \ldots
    \end{bmatrix},
\end{eqnarray*}
\begin{eqnarray*}
    E^{\top}W_{KQ,1}P=\sqrt{\frac{p}{d}}\begin{bmatrix}
        v_x x_1^{\top}p_{11}^x & v_x x_1^{\top}p_{12}^x & v_x x_1^{\top}p_{21}^x & \ldots\\
        v_y y_1 p_{11}^y &  v_y y_1 p_{12}^y &  v_y y_1 p_{21}^y & \ldots\\
        \ldots
    \end{bmatrix},
\end{eqnarray*}
and
\begin{eqnarray*}
    P^{\top}W_{KQ,1}P=\begin{bmatrix}
        v_x \|p_{11}^x\|^2+v_y (p_{11}^y)^2 &\ldots\\\ldots
    \end{bmatrix}=c^2\begin{bmatrix}
        v_x & v_x & 0 & 0 & 0 & 0 & \ldots & 0\\
        v_x & v_x & 0 & 0 & 0 & 0 & \ldots & 0\\
       0& 0 & v_x & v_x & 0 & 0 & \ldots & 0\\
        0 & 0& v_x & v_x & 0 & 0 & \ldots & 0\\
        0 & 0 & 0 & 0 & v_x & v_x & \ldots & 0\\
        \ldots\\
        0 & 0 & 0 & 0 & 0 & 0 & \ldots & 0
    \end{bmatrix}.
\end{eqnarray*}

    For $E^{\top}W_{KQ,1}E$, the moment generating function of $x_i^{\top}x_j$ is
    \begin{eqnarray*}
        \mathbb{E}\exp(tx_i^{\top}x_j)=\mathbb{E}\exp(t^2\|x_j\|^2/2)=\frac{1}{\sqrt{2\pi}}\int\exp\left(-\frac{\|x_j\|^2}{2}(1-t^2)\right)dx_j=\frac{1}{(1-t)^{d/2}}.
    \end{eqnarray*}
    Using Chernoff bounds, we have for any $z>0$,
    \begin{eqnarray*}
        P(x_i^{\top}x_j>z)\leq \min_{t>0}\frac{\mathbb{E}\exp(tx_i^{\top}x_j)}{\exp(tz)}=\min_{t>0}\frac{1}{(1-t)^{d/2}}\frac{1}{\exp(tz)}\leq\exp\left(\frac{d}{2}\log 2-\frac{z}{2}\right).
    \end{eqnarray*}
    For $E^{\top}W_{KQ,1}P$, similarly, we know that
    \begin{eqnarray*}
        P(cx_{i1}>z)\leq \min_{t>0}\frac{\mathbb{E}\exp(tcx_{i1})}{\exp(tz)}\leq \exp\left( \frac{c^2}{8}-\frac{z}{2} \right).
    \end{eqnarray*}
    Therefore, for the attention score associated with the columns $O_{:,2i}$s, we have
    \begin{eqnarray*}
      &&\sup_{i} P\left(\sum_{j< i}\exp(v_xx_i^{\top}x_j+v_x x_j^{\top}p_{i2}^x)+i>\exp(v_xc^2) \right)\\
      &\leq & d\sup_{j}P\left(d\exp(v_xx_i^{\top}x_j+v_x x_j^{\top}p_{i2}^x)>\exp(v_xc_0^2d\log(d)) \right)(1+o(1))\\
      &\leq & d\sup_{j}P\left(2d\exp(v_xx_i^{\top}x_j)>\exp(v_xc_0^2d\log(d)) \right)(1+o(1))\\
      &&+d\sup_{j}P\left(2d\exp(v_x x_j^{\top}p_{i2}^x)>\exp(v_xc_0^2d\log(d)) \right)(1+o(1))\\
      &\leq&d^2\left[\exp\left(\frac{d}{2}\log 2-\frac{c_0^2d}{2}\log(d)\right)+\exp\left( \frac{c_0^2d\log(d)}{8}- \frac{c_0^2d\log(d)}{2} \right)\right](1+o(1)).
    \end{eqnarray*}
    While the above shows that $\exp(v_xc^2)$ dominates the sum of all the other attention scores in the same column, we also have $\exp(v_x\|x_i\|^2+v_xx_i^{\top}p_{i2}^x+v_xc^2)=\exp(v_xc^2)(1+o(1))$.
    As a result, for the columns $O_{:,2i}$s, the attention score for $(p_{i1}^x+x_i,0)$ and $(p_{i2}^x,y_i)$ will be both $1/2+o$, and all the other attention scores will be negligible.

\end{proof}
\subsection{Theorem \ref{prop:large_D}}

\begin{proof}[Proof of Theorem \ref{prop:large_D}]
    
From the definition of $E$ and $P$, we have
\begin{eqnarray*}
    &&E\phi\left( (E+P)^{\top}W_{KQ,1}(E+P)+M  \right)_{:,2i}\\
    &=&E\phi\left( (E+P)^{\top}W_{KQ,1}(E+P)_{:,2i}+M_{:,2i}  \right)\\
    &=&\frac{1}{B_i}\sum_{j\leq i}  \begin{bmatrix}
        x_j+p_{j1}^x\\
        0
    \end{bmatrix} \exp\left( v_x ( x_j+p_{j1}^x )^{\top}(p_{i2}^x)\right)+\frac{1}{B_i}\sum_{j<i} \begin{bmatrix}
        p_{j2}^x\\
        y_j
    \end{bmatrix}+\frac{1}{B_i}\begin{bmatrix}
        p_{i2}^x\\y_i
    \end{bmatrix}\exp(v_x\|p_{i1}\|^2)
\end{eqnarray*}
where
$$B_i=\sum_{j<i}\left[\exp(v_x x_j^{\top}p_{i2}^x)+\exp(v_x (p_{j2}^x)^{\top}p_{i1}^x )\right]+\exp(v_x x_i^{\top}p_{i2}^x+v_x c^2)+\exp(v_x c^2).$$
When $i\gg d$, we have
\begin{eqnarray*}
    &&\sum_{j\leq i}  \begin{bmatrix}
        x_j+p_{j1}^x\\
        0
    \end{bmatrix} \exp\left( v_x ( x_j+p_{j1}^x )^{\top}(p_{i2}^x)\right)\\
    &=&\sum_{j< i}\mathbb{E}  \begin{bmatrix}
        x_j+p_{j1}^x\\
        0
    \end{bmatrix} \exp\left( v_x ( x_j+p_{j1}^x )^{\top}(p_{i2}^x)\right)+\begin{bmatrix}
        x_i+p_{i1}^x\\0
    \end{bmatrix}\exp\left(v_x(x_i+p_{i1}^x)^{\top}p_{i2}^x\right)\\
    &&+\sum_{j< i}  \begin{bmatrix}
        x_j+p_{j1}^x\\
        0
    \end{bmatrix} \exp\left( v_x ( x_j+p_{j1}^x )^{\top}(p_{i2}^x)\right)\\
    &&-\sum_{j< i} \mathbb{E} \begin{bmatrix}
        x_j+p_{j1}^x\\
        0
    \end{bmatrix} \exp\left( v_x ( x_j+p_{j1}^x )^{\top}(p_{i2}^x)\right)\\
    &=&\underbrace{\sum_{j< i}  \begin{bmatrix}
        v_xp_{i2}^x+p_{j1}^x\\
        0
    \end{bmatrix} \exp\left(\frac{v_x^2}{2}\|p_{i2}^x\|^2+ v_x( p_{j1}^x )^{\top}p_{i2}^x\right)}_{:=\zeta_1}+\underbrace{\begin{bmatrix}
        x_i+p_{i1}^x\\0
    \end{bmatrix}\exp\left(v_x(x_i+p_{i1}^x)^{\top}p_{i2}^x\right)}_{:=\zeta_2}+\underbrace{O_p(\sqrt{di})}_{:=\zeta_3}.
\end{eqnarray*}

For simplicity, we assume $\|p_{jk}^x\|$s are all the same for $j=1,\ldots,D$ and $k=1,2$, and $\|p_{jk}^x\|\leq \sqrt{d}$. 

For $\zeta_2$, we have $x_i^{\top}p_{i2}^x=O_p(\|p_{i2}^x\|)$. When $i=\Theta(D)$,

\begin{enumerate}
    \item If $\|p_{jk}^x\|=o(1)$, then $\zeta_3$ will dominate $\zeta_2$.
    \item If $\|p_{jk}^x\|\gg 1$ and $v_x\gg 1$, then $\zeta_1$ will dominate $\zeta_2$ when $D\gg d^2$.
    \item If $\|p_{jk}^x\|\gg 1$ and $v_x=o(1)$, then taking $p_{i1}^x=p_{i2}^x$,
    $$\zeta_2=\begin{bmatrix}
        x_i+p_{i1}^x\\0
    \end{bmatrix}\exp(v_x(p_{i1}^x)^{\top}p_{i2}^x)+o.$$
    When $D\gg \exp(d)$, $\zeta_3$ will dominate $\zeta_2$.
\end{enumerate}

\end{proof}

\subsection{Theorem \ref{prop:icl_emb}}
\begin{proof}
From the definition of $E$ and $P$,  we obtain

\begin{eqnarray}
    &&W_{in}E\phi\left( (W_{in}E+P)^{\top}W_{KQ,1}(W_{in}E+P)+M  \right)_{:,2i-1}\nonumber\\
    &=&W_{in}E\phi\left( (W_{in}E+P)^{\top}W_{KQ,1}(W_{in}E+P)_{:,2i-1}+M_{:,2i-1}  \right)\nonumber.
\end{eqnarray}  
Based on the above, we assume $W_{KQ,1}$ and $W_{in}$ satisfy that
\begin{eqnarray*}
    W_{in}^{\top}W_{KQ,1}W_{in}=\begin{bmatrix}
        v_x I_d & 0 \\ 0 & 0
    \end{bmatrix},
\end{eqnarray*}
and all the singular values of $W_{in}$ are the same as $\kappa\sqrt{p/d}$ for some positive number $\kappa$. In addition, denote the original PE design in Theorem \ref{thm:icl_no_emb} as $P_0$ with columns $p_{jk}^0$, then take
\begin{eqnarray*}
    P=\sqrt{v_x}W_{KQ,1}^{-1/2}P_0.
\end{eqnarray*}
Then we obtain
\begin{eqnarray*}
    &&W_{in}E\phi\left( (W_{in}E+P)^{\top}W_{KQ,1}(W_{in}E+P)+M  \right)_{:,2i-1}\\
    &=&\frac{1}{A_i}\underbrace{\sum_{j\leq i}  \left(W_{in}\begin{bmatrix}
        x_j\\
        0
    \end{bmatrix}+p_{j1}\right) \exp\left( v_x x_i^{\top}x_j + v_x (p_{j1}^0)^{\top}p_{i1}^0+ \begin{bmatrix}
        x_i & 0
    \end{bmatrix}W_{in}^{\top}W_{KQ,1}^{1/2}p_{j1}^0+\begin{bmatrix}
        x_j & 0
    \end{bmatrix}W_{in}^{\top}W_{KQ,1}^{1/2}p_{i1}^0 \right)}_{:=\xi_1}\\
    &&+\frac{1}{A_i}\underbrace{\sum_{j<i} \left(W_{in}\begin{bmatrix}
        0\\
        y_j
    \end{bmatrix}+p_{j2}\right)}_{:=\xi_2},\label{eqn:token_x}
\end{eqnarray*}
where $$A_i=\sum_{j\leq i} \exp\left( v_x^2 ( x_j+p_{j1}^x )^{\top}(x_i+p_{i1}^x)\right)+\sum_{j<i}\exp(v_x^2 (p_{j2}^x)^{\top}p_{i1}^x ).$$

Below we calculate the expectation and the variance of $\xi_1$. When fixing $\theta$ and $x_i$
\begin{eqnarray*}
    &&\mathbb{E}\xi_1\\
    &=&\sum_{j<i}\mathbb{E}\left(W_{in}\begin{bmatrix}
        x_j\\
        0
    \end{bmatrix}+p_{j1}\right) \exp\left( v_x x_i^{\top}x_j + v_x (p_{j1}^0)^{\top}p_{i1}^0+ \sqrt{v_x}\begin{bmatrix}
        x_i & 0
    \end{bmatrix}W_{in}^{\top}W_{KQ,1}^{1/2}p_{j1}^0+\sqrt{v_x}\begin{bmatrix}
        x_j & 0
    \end{bmatrix}W_{in}^{\top}W_{KQ,1}^{1/2}p_{i1}^0 \right)\\
    &&+\left(W_{in}\begin{bmatrix}
        x_i\\
        0
    \end{bmatrix}+p_{j1}\right) \exp\left( v_x \|x_i\|^2 + v_x c^2 + 2\sqrt{v_x}\begin{bmatrix}
        x_i & 0
    \end{bmatrix}W_{in}^{\top}W_{KQ,1}^{1/2}p_{i1}^0\right)\\
    &=&\sum_{j<i}\mathbb{E}W_{in}\begin{bmatrix}
        x_j\\
        0
    \end{bmatrix}\exp\left( v_x x_i^{\top}x_j + v_x (p_{j1}^0)^{\top}p_{i1}^0+ \sqrt{v_x}\begin{bmatrix}
        x_i & 0
    \end{bmatrix}W_{in}^{\top}W_{KQ,1}^{1/2}p_{j1}^0+\sqrt{v_x}\begin{bmatrix}
        x_j & 0
    \end{bmatrix}W_{in}^{\top}W_{KQ,1}^{1/2}p_{i1}^0 \right)\\
    &&+\sum_{j<i}\mathbb{E}W_{in}p_{j1}\exp\left( v_x x_i^{\top}x_j + v_x (p_{j1}^0)^{\top}p_{i1}^0+\sqrt{v_x} \begin{bmatrix}
        x_i & 0
    \end{bmatrix}W_{in}^{\top}W_{KQ,1}^{1/2}p_{j1}^0+\sqrt{v_x}\begin{bmatrix}
        x_j & 0
    \end{bmatrix}W_{in}^{\top}W_{KQ,1}^{1/2}p_{i1}^0 \right)\\
    &&+\left(W_{in}\begin{bmatrix}
        x_i\\
        0
    \end{bmatrix}+p_{j1}\right) \exp\left( v_x \|x_i\|^2 + v_x c^2 + 2\sqrt{v_x}\begin{bmatrix}
        x_i & 0
    \end{bmatrix}W_{in}^{\top}W_{KQ,1}^{1/2}p_{i1}^0\right)\\
    &=&\sum_{j<i}v_x\left(W_{in}\begin{bmatrix}
         x_i\\
        0
    \end{bmatrix}+W_{in}^{\top}W_{KQ,1}^{1/2}p_{j1}^0\right)\\
    &&\qquad\qquad\qquad\times\exp\left( \frac{1}{2}\left\|v_x x_i+ v_x (W_{in}^{\top}W_{KQ,1}^{1/2}p_{i1}^0)_{1:d}\right\|^2+ v_x (p_{j1}^0)^{\top}p_{i1}^0 +\sqrt{v_x}\begin{bmatrix}
        x_i & 0
    \end{bmatrix}W_{in}^{\top}W_{KQ,1}^{1/2}p_{j1}^0 \right)\\
    &&+\sum_{j<i}W_{in}p_{j1}\exp\left( \frac{1}{2}\left\|v_x x_i+ v_x (W_{in}^{\top}W_{KQ,1}^{1/2}p_{i1}^0)_{1:d}\right\|^2+ v_x (p_{j1}^0)^{\top}p_{i1}^0 +\sqrt{v_x}\begin{bmatrix}
        x_i & 0
    \end{bmatrix}W_{in}^{\top}W_{KQ,1}^{1/2}p_{j1}^0 \right)\\
    &&+\left(W_{in}\begin{bmatrix}
        x_i\\
        0
    \end{bmatrix}+p_{j1}\right) \exp\left( v_x \|x_i\|^2 + v_x c^2 + 2\sqrt{v_x}\begin{bmatrix}
        x_i & 0
    \end{bmatrix}W_{in}^{\top}W_{KQ,1}^{1/2}p_{i1}^0\right)
\end{eqnarray*}
For the variance, we have
\begin{eqnarray*}
    &&\mathbb{E}\xi_1\xi_1^{\top}\\
    &=&\sum_{j\neq k}\mathbb{E}\left(W_{in}\begin{bmatrix}
        x_j\\0
    \end{bmatrix} +p_{j1}\right)\left(W_{in}\begin{bmatrix}
        x_k\\0
    \end{bmatrix} +p_{k1}\right)^{\top}\\
    &&\qquad\qquad\qquad \times  \exp\left( v_x x_i^{\top}x_j + v_x (p_{j1}^0)^{\top}p_{i1}^0+ \sqrt{v_x}\begin{bmatrix}
        x_i & 0
    \end{bmatrix}W_{in}^{\top}W_{KQ,1}^{1/2}p_{j1}^0+\sqrt{v_x}\begin{bmatrix}
        x_j & 0
    \end{bmatrix}W_{in}^{\top}W_{KQ,1}^{1/2}p_{i1}^0 \right)\\
    &&\qquad\qquad\qquad\times  \exp\left( v_x x_i^{\top}x_k + v_x (p_{k1}^0)^{\top}p_{i1}^0+ \sqrt{v_x}\begin{bmatrix}
        x_i & 0
    \end{bmatrix}W_{in}^{\top}W_{KQ,1}^{1/2}p_{j1}^0+\sqrt{v_x}\begin{bmatrix}
        x_k & 0
    \end{bmatrix}W_{in}^{\top}W_{KQ,1}^{1/2}p_{i1}^0 \right)\\
    &&+\sum_{j<i}\mathbb{E}\left(W_{in}\begin{bmatrix}
        x_j\\0
    \end{bmatrix} +p_{j1}\right)\left(W_{in}\begin{bmatrix}
        x_j\\0
    \end{bmatrix} +p_{j1}\right)^{\top} \\
    &&\qquad\qquad\qquad \times  \exp^2\left( v_x x_i^{\top}x_j + v_x (p_{j1}^0)^{\top}p_{i1}^0+ \sqrt{v_x}\begin{bmatrix}
        x_i & 0
    \end{bmatrix}W_{in}^{\top}W_{KQ,1}^{1/2}p_{j1}^0+\sqrt{v_x}\begin{bmatrix}
        x_j & 0
    \end{bmatrix}W_{in}^{\top}W_{KQ,1}^{1/2}p_{i1}^0 \right)\\
    &&+2\sum_{j<i}\mathbb{E}\left(W_{in}\begin{bmatrix}
        x_i\\
        0
    \end{bmatrix}+p_{i1}\right) \left(W_{in}\begin{bmatrix}
        x_j\\0
    \end{bmatrix} +p_{j1}\right)^{\top}W_{in}^{\top} \\
    &&\qquad\qquad\qquad\times  \exp\left( v_x x_i^{\top}x_j + v_x (p_{j1}^0)^{\top}p_{i1}^0+ \sqrt{v_x}\begin{bmatrix}
        x_i & 0
    \end{bmatrix}W_{in}^{\top}W_{KQ,1}^{1/2}p_{j1}^0+\sqrt{v_x}\begin{bmatrix}
        x_j & 0
    \end{bmatrix}W_{in}^{\top}W_{KQ,1}^{1/2}p_{i1}^0 \right)\\
     &&\qquad\qquad\qquad\times \exp\left( v_x \|x_i\|^2 + v_x c^2 + 2\sqrt{v_x}\begin{bmatrix}
        x_i & 0
    \end{bmatrix}W_{in}^{\top}W_{KQ,1}^{1/2}p_{i1}^0\right)\\
    &&+\left(W_{in}\begin{bmatrix}
        x_i\\
        0
    \end{bmatrix}+p_{i1}\right) \left(W_{in}\begin{bmatrix}
        x_j\\0
    \end{bmatrix} +p_{j1}\right)^{\top}\\
    &&\qquad\qquad\qquad\times \exp^2\left( v_x \|x_i\|^2 + v_x c^2 + 2\sqrt{v_x}\begin{bmatrix}
        x_i & 0
    \end{bmatrix}W_{in}^{\top}W_{KQ,1}^{1/2}p_{i1}^0\right).
\end{eqnarray*}
Then
\begin{eqnarray*}
    &&\mathbb{E}\xi_1\xi_1^{\top}-\mathbb{E}\xi_1\mathbb{E}\xi_1^{\top}\\
    &=&\sum_{j<i}\mathbb{E}\left(W_{in}\begin{bmatrix}
        x_j\\0
    \end{bmatrix} +p_{j1}\right)\left(W_{in}\begin{bmatrix}
        x_j\\0
    \end{bmatrix} +p_{j1}\right)^{\top}\\
    &&\qquad\qquad\qquad \times  \exp^2\left( v_x x_i^{\top}x_j + v_x (p_{j1}^0)^{\top}p_{i1}^0+ \sqrt{v_x}\begin{bmatrix}
        x_i & 0
    \end{bmatrix}W_{in}^{\top}W_{KQ,1}^{1/2}p_{j1}^0+\sqrt{v_x}\begin{bmatrix}
        x_j & 0
    \end{bmatrix}W_{in}^{\top}W_{KQ,1}^{1/2}p_{i1}^0 \right)\\
    &&-\sum_{j<i}\mathbb{E}\left(W_{in}\begin{bmatrix}
        x_j\\0
    \end{bmatrix} +p_{j1}\right)\\
    &&\qquad\qquad\qquad \times \exp\left( v_x x_i^{\top}x_j + v_x (p_{j1}^0)^{\top}p_{i1}^0+ \sqrt{v_x}\begin{bmatrix}
        x_i & 0
    \end{bmatrix}W_{in}^{\top}W_{KQ,1}^{1/2}p_{j1}^0+\sqrt{v_x}\begin{bmatrix}
        x_j & 0
    \end{bmatrix}W_{in}^{\top}W_{KQ,1}^{1/2}p_{i1}^0 \right)\\
    &&\qquad\times \mathbb{E}\left(W_{in}\begin{bmatrix}
        x_j\\0
    \end{bmatrix} +p_{j1}\right)^{\top} \\
    &&\qquad\qquad\qquad \times\exp\left( v_x x_i^{\top}x_j + v_x (p_{j1}^0)^{\top}p_{i1}^0+ \sqrt{v_x}\begin{bmatrix}
        x_i & 0
    \end{bmatrix}W_{in}^{\top}W_{KQ,1}^{1/2}p_{j1}^0+\sqrt{v_x}\begin{bmatrix}
        x_j & 0
    \end{bmatrix}W_{in}^{\top}W_{KQ,1}^{1/2}p_{i1}^0 \right).
\end{eqnarray*}

Consider the expectation of $Wxx^{\top}W^{\top}\exp(x^{\top}W^{\top}z)$, it becomes
\begin{eqnarray*}
   \mathbb{E}Wxx^{\top}W^{\top}\exp(2x^{\top}W^{\top}z)=WW^{\top}\exp\left(2\|W^{\top}z\|^2\right) +4WW^{\top}zz^{\top}WW^{\top}\exp\left(2\|W^{\top}z\|^2\right).
\end{eqnarray*}
Meanwhile,
\begin{eqnarray*}
    \mathbb{E}Wx\exp(x^{\top}W^{\top}z)\mathbb{E}x^{\top}W^{\top}\exp(x^{\top}W^{\top}z)=WW^{\top}zz^{\top}WW^{\top}\exp(\|W^\top z\|^2).
\end{eqnarray*}
The notation $z$ in the above two equations represents
\begin{eqnarray*}
    z=W_{KQ,1}W_{in}x_i+\sqrt{v_x}W_{KQ,1}^{1/2}p_{i1}^0,
\end{eqnarray*}
and
\begin{eqnarray*}
    W^{\top}_{in}z=v_x x_i+\sqrt{v_x}W_{in}^{\top}W_{KQ,1}^{1/2}p_{i1}^0.
\end{eqnarray*}
We know that all the singular values of $W_{in}$ are $\kappa\sqrt{p/d}$, thus
\begin{eqnarray*}
    \|W_{in}W^{\top}_{in}\|_F^2=\kappa^2 p,\quad \|W_{in}W^{\top}_{in}zz^{\top}W_{in}W^{\top}_{in}\|_F^2=\kappa^2\frac{p}{d}\|v_x x_i+\sqrt{v_x}W_{in}^{\top}W_{KQ,1}^{1/2}p_{i1}^0\|^2 .
\end{eqnarray*}
When $c\gg\sqrt{d}\log p$, we have $\xi_1$ concentrates on the $x_i$ term, i.e.,
\begin{eqnarray*}
    \xi_1 =\left( \left(W_{in}\begin{bmatrix}
        x_i\\
        0
    \end{bmatrix}+p_{j1}\right) \exp\left( v_x \|x_i\|^2 + v_x c^2 + 2\sqrt{v_x}\begin{bmatrix}
        x_i & 0
    \end{bmatrix}W_{in}^{\top}W_{KQ,1}^{1/2}p_{i1}^0\right)\right)(1+o(1)).
\end{eqnarray*}

On the other hand, since $A_i\rightarrow\mathbb{E}A_i$, we also have $$\frac{1}{A_i}-\frac{1}{\mathbb{E}A_i}=O_p\left(\frac{1}{\sqrt{i}}\right).$$
Therefore,
\begin{eqnarray*}
    &&W_{in}E\phi\left( (W_{in}E+P)^{\top}W_{KQ,1}(W_{in}E+P)+M  \right)_{:,2i-1}\\
    &=&\frac{1}{\mathbb{E}A_i}\left( \left(W_{in}\begin{bmatrix}
        x_i\\
        0
    \end{bmatrix}+p_{j1}\right) \exp\left( v_x \|x_i\|^2 + v_x c^2 + 2\sqrt{v_x}\begin{bmatrix}
        x_i & 0
    \end{bmatrix}W_{in}^{\top}W_{KQ,1}^{1/2}p_{i1}^0\right)\right) +o\\
    &=&\left(W_{in}\begin{bmatrix}
        x_i\\
        0
    \end{bmatrix}+p_{j1}\right)+o,
\end{eqnarray*}
and
\begin{eqnarray*}
    O_{:,2i-1}=(v+1)\left(W_{in}\begin{bmatrix}
        x_i\\
        0
    \end{bmatrix}+p_{j1}\right)+o.
\end{eqnarray*}

In terms of $O_{:,2i}$, similarly,

\begin{eqnarray*}
    O_{:,2i}=\frac{v}{2}\left(W_{in}\begin{bmatrix}
        x_j\\y_j
    \end{bmatrix} +p_{j1}+p_{j2}\right)+\left(W_{in}\begin{bmatrix}
        0\\y_j
    \end{bmatrix} +p_{j2}\right)+o.
\end{eqnarray*}

\end{proof}

\subsection{Theorem \ref{thm:icl_prediction}}
\begin{proof}[Proof of Theorem \ref{thm:icl_prediction}]
    When \begin{eqnarray*}
    E_{:,2i-1}=(v+1)\left(W_{in}\begin{bmatrix}
        x_i\\
        0
    \end{bmatrix}+p_{i}\right),
\end{eqnarray*}
and
\begin{eqnarray*}
    E_{:,2i}=\frac{v}{2}W_{in}\begin{bmatrix}
        x_i\\0
    \end{bmatrix} +\left( \frac{v}{2}+1 \right)W_{in}\begin{bmatrix}
        0\\y_i
    \end{bmatrix} +(v+1)p_{i},
\end{eqnarray*}
one can obtain that
\begin{eqnarray*}
    &&EE^{\top}\\
    &=& (v+1)^2\sum\left(W_{in}\begin{bmatrix}
        x_i\\
        0
    \end{bmatrix}+p_{i}\right)\left(W_{in}\begin{bmatrix}
        x_i\\
        0
    \end{bmatrix}+p_{i}\right)^{\top}\\
    &&+\sum \left(\frac{v}{2}W_{in}\begin{bmatrix}
        x_i\\0
    \end{bmatrix} +\left( \frac{v}{2}+1 \right)W_{in}\begin{bmatrix}
        0\\y_i
    \end{bmatrix} +(v+1)p_{i}\right)\left(\frac{v}{2}W_{in}\begin{bmatrix}
        x_i\\0
    \end{bmatrix} +\left( \frac{v}{2}+1 \right)W_{in}\begin{bmatrix}
        0\\y_i
    \end{bmatrix} +(v+1)p_{i}\right)^{\top}\\
    &&+(v+1)\sum\left(W_{in}\begin{bmatrix}
        x_i\\
        0
    \end{bmatrix}+p_{i}\right)\left(\frac{v}{2}W_{in}\begin{bmatrix}
        x_i\\0
    \end{bmatrix} +\left( \frac{v}{2}+1 \right)W_{in}\begin{bmatrix}
        0\\y_i
    \end{bmatrix} +(v+1)p_{i}\right)^{\top}\\
    &&+(v+1)\sum\left(\frac{v}{2}W_{in}\begin{bmatrix}
        x_i\\0
    \end{bmatrix} +\left( \frac{v}{2}+1 \right)W_{in}\begin{bmatrix}
        0\\y_i
    \end{bmatrix} +(v+1)p_{i}\right)\left(W_{in}\begin{bmatrix}
        x_i\\
        0
    \end{bmatrix}+p_{i}\right)^{\top}\\
    &=&\left((v+1)^2+\frac{v^2}{4}+v(v+1)\right)\underbrace{W_{in}\left(\sum \begin{bmatrix}
        x_ix_i^{\top} & 0\\
        0 & 0
    \end{bmatrix}\right)W_{in}^{\top}}_{:=C_1}\\
    &&+\left( 2(v+1)^2+v(v+1) \right)\underbrace{\sum W_{in}\begin{bmatrix}
        x_i\\0
    \end{bmatrix}p_i^{\top}+p_i\begin{bmatrix}
        x_i\\0
    \end{bmatrix}^{\top}W_{in}^{\top}}_{:=C_2}\\
    &&+4(v+1)^2\underbrace{\sum p_ip_i^{\top}}_{:=C_3}+\left(\frac{v}{2}+1(v+1)\left(\frac{v}{2}+1\right)\right)\underbrace{\sum W_{in}\begin{bmatrix}
        0\\y_i
    \end{bmatrix}p_i^{\top}+p_i\begin{bmatrix}
        0\\y_i
    \end{bmatrix}^{\top}W_{in}^{\top}}_{:=C_4}\\
    &&+\left(\left(\frac{v}{2}+1\right)^2+\left(\frac{v}{2}+1\right)(v+1)\right)\underbrace{W_{in}\left(\sum \begin{bmatrix}
        0 & y_ix_i\\
        y_ix_i^{\top} & 0
    \end{bmatrix}\right)W_{in}^{\top}}_{:=C_5}\\
    &:=& v_1C_1+v_2C_2+v_3C_3+v_4C_4+v_5C_5.
\end{eqnarray*}
One can see that $C_2$ and $C_4$ negligible compared to $C_1$, $C_3$, and $C_5$. In addition,
\begin{eqnarray*}
    C_1&=&W_{in}\left(\sum \begin{bmatrix}
        x_ix_i^{\top} & 0\\
        0 & 0
    \end{bmatrix}\right)W_{in}^{\top}-W_{in}\mathbb{E}\left(\sum \begin{bmatrix}
        x_ix_i^{\top} & 0\\
        0 & 0
    \end{bmatrix}\right)W_{in}^{\top}\\
    &&+W_{in}\mathbb{E}\left(\sum \begin{bmatrix}
        x_ix_i^{\top} & 0\\
        0 & 0
    \end{bmatrix}\right)W_{in}^{\top}\\
    &=&DW_{in}W_{in}^{\top}\left(1+O_p\left(\sqrt{\frac{d}{D}}\right)\right),
\end{eqnarray*}
and \begin{eqnarray*}
    C_5=DW_{in}\begin{bmatrix}
        0 & \theta\\
        \theta^{\top} & 0
    \end{bmatrix}W_{in}\left(1+O_p\left(\sqrt{\frac{d}{D}}\right)\right).
\end{eqnarray*}

Given the input format of $E$, the prediction of $y_q$ satisfies
\begin{eqnarray*}
    \widehat y_q &=&  w_v^{\top}\left( EE^{\top}\left(W_{in}\begin{bmatrix}
        x_q\\0
    \end{bmatrix}+p_{D+1}\right)\right)\\
    &=&w_v^{\top}\left( v_1 D W_{in}W_{in}^{\top}+v_3\sum p_ip_i^{\top} + v_5DW_{in}\begin{bmatrix}
        0 & \theta\\
        \theta^{\top} & 0
    \end{bmatrix}W_{in}^{\top}\right)\left(W_{in}\begin{bmatrix}
        x_q\\0
    \end{bmatrix}+p_{D+1}\right),
\end{eqnarray*}
for some $w_v$.

Taking $$W_{in}=\begin{bmatrix}
    1 & 0 & 0 & \ldots\\ 
    1 & 0 & 0 & \ldots\\
    \ldots\\
    0 & 1 & 0 & \ldots\\
    0 & 1 & 0 & \ldots\\
    \ldots\\
    0 & 0 & 1 & \ldots\\
    0 & 0 & 1 & \ldots\\
    \ldots
\end{bmatrix},$$
and denoting $m=p/d$, we have
\begin{eqnarray*}
    W_{in}^{\top}W_{in}=m I_d,
\end{eqnarray*}
thus
\begin{eqnarray*}
&&w_v^{\top}\left( v_1 D W_{in}W_{in}^{\top}+v_3\sum p_ip_i^{\top}+v_5DW_{in}\begin{bmatrix}
        0 & \theta\\
        \theta^{\top} & 0
    \end{bmatrix}W_{in}^{\top} \right)W_{in}\begin{bmatrix}
        x_q\\0
    \end{bmatrix}\\
&=&w_v^{\top}\left( v_1mDW_{in} + v_3 c^2\begin{bmatrix}
    I_{D}& 0\\
    0 & 0
\end{bmatrix}W_{in}+v_5 mDW_{in}\begin{bmatrix}
    0 & \theta\\
    \theta^{\top} & 0
\end{bmatrix}  \right)  \begin{bmatrix}
    x_q\\0
\end{bmatrix}\\
&=&w_v^{\top}\left( v_1mDW_{in} \begin{bmatrix}
    x_q\\0
\end{bmatrix} + v_3 c^2\begin{bmatrix}
    I_{D}& 0\\
    0 & 0
\end{bmatrix}W_{in} \begin{bmatrix}
    x_q\\0
\end{bmatrix}+v_5 mDW_{in}\begin{bmatrix}
    0 \\ y_q
\end{bmatrix}  \right).
\end{eqnarray*}
When taking 
$$w_v^{\top}=\frac{1}{v_5 m D} \begin{bmatrix}
    0 & 0 & \ldots & 0 & 1
\end{bmatrix},$$
we have
\begin{eqnarray*}
    w_v^{\top}W_{in}=\begin{bmatrix}
        0 & 0 & \ldots & 0 & 1
    \end{bmatrix},
\end{eqnarray*}
and
\begin{eqnarray*}
    w_v^{\top}\left( v_1 D W_{in}W_{in}^{\top}+v_3\sum p_ip_i^{\top}+v_5DW_{in}\begin{bmatrix}
        0 & \theta\\
        \theta^{\top} & 0
    \end{bmatrix}W_{in}^{\top} \right)W_{in}\begin{bmatrix}
        x_q\\0
    \end{bmatrix}=y_q.
\end{eqnarray*}
On the other hand,
\begin{eqnarray*}
    &&w_v^{\top}\left( v_1 D W_{in}W_{in}^{\top}+v_3\sum p_ip_i^{\top} + v_5DW_{in}\begin{bmatrix}
        0 & \theta\\
        \theta^{\top} & 0
    \end{bmatrix}W_{in}^{\top}\right)p_{D+1}\\
    &=&w_v^{\top} c\left( v_1 D W_{in}e_{\lceil(D+1)/m\rceil} + v_5DW_{in}\begin{bmatrix}
        0 & \theta\\
        \theta^{\top} & 0
    \end{bmatrix}e_{\lceil(D+1)/m\rceil}\right)\\
    &=& 0.
\end{eqnarray*}
That is, the PE does not results in any bias in the ICL prediction in the second layer, and $\mathbb{E}(\widehat y_q-y_q)^2=O(d/D)$.
\end{proof}

\newpage
\section{Configuration Details}\label{sec:appendix:configuration}

    
Table \ref{tab:code} below shows how we implement the changes in the components of the transformer.

\begin{table}[!ht]
    \centering
    \begin{tabular}{c|c}
     Purpose & Change \\\hline
        Change number of layers. & Change GPT2 configuration.\\
        Change input embedding. & Change the configuration in \cite{garg2022can}.\\
         Change PE. & Change \texttt{wpe} in the GPT2 model.  \\
         Remove attention mask. & Set all elements to 1 in \texttt{bias} in the attention layer. \\
         Set $W_{KQ,1}$ as zero. & Force the first 2/3 rows of \texttt{c\_attn.weights} and intercept in the attention as zero.
    \end{tabular}
    \caption{Instruction in controlling the components in the transformer.}
    \label{tab:code}
\end{table}

\section{Other Additional Experiment Results}\label{sec:appendix:other}


\begin{figure}[!ht]
    \centering
    \includegraphics[width=0.8\linewidth]{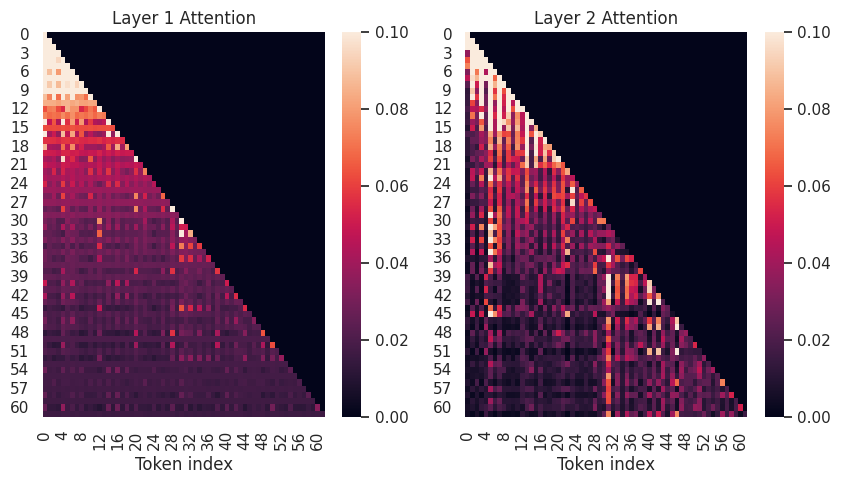}\vspace{-0.1in}
    \caption{Full attention scores on single head, two layer, with mask, no PE, $E_2$ format. There are 30+1 pairs of x and y, thus 31 * 2 = 62 tokens in the prompt. Only attentions score from one example is shown. Each rows is the attention of one token.}
    \label{fig:attn_no_pos_full}
\end{figure}
\begin{figure}[!ht]
    \centering
    \includegraphics[width=0.8\linewidth]{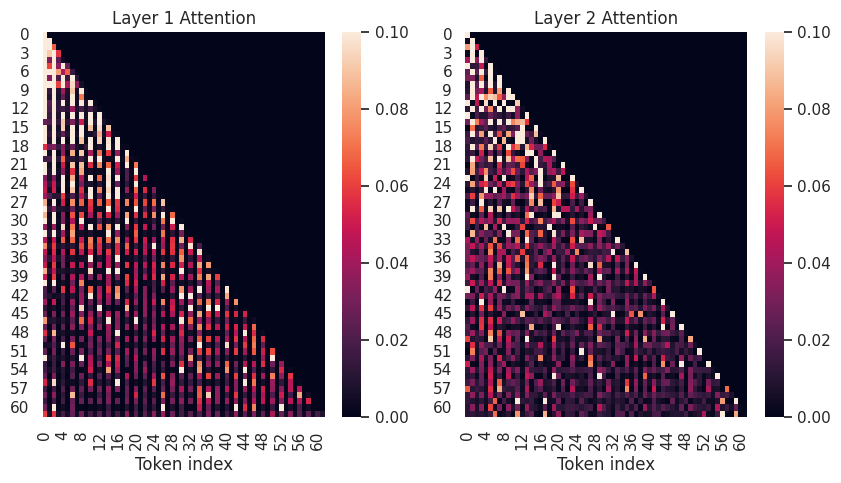}\vspace{-0.1in}
    \caption{Full attention scores on single head, two layer, with mask, with PE, $E_2$ format.}
    \label{fig:attn_pos_full}
\end{figure}

\newpage
\section{Additional Simulations}\label{sec:appendix:simulation}

\subsection{Ten Repetitions}
Figure \ref{fig:10_repetition} shows the ICL performance of 10 different repetitions under each setting.
\begin{figure}[!ht]
    \centering
    \includegraphics[scale=0.55]{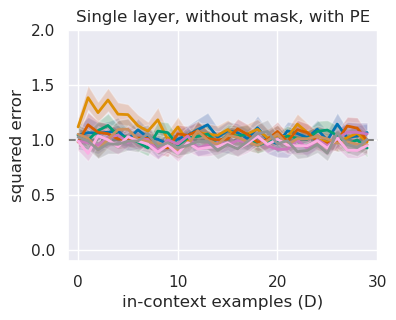}
    \includegraphics[scale=0.55]{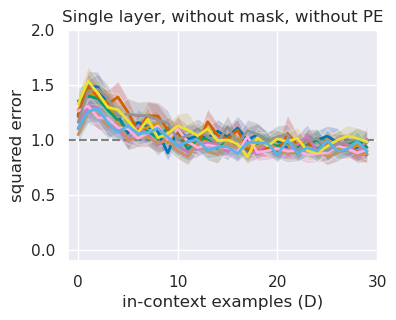}
    \includegraphics[scale=0.55]{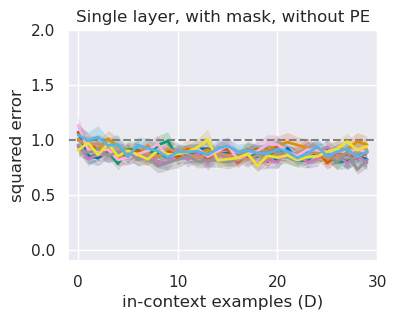}
    \includegraphics[scale=0.55]{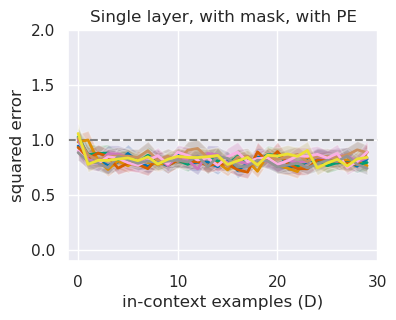}
    \includegraphics[scale=0.55]{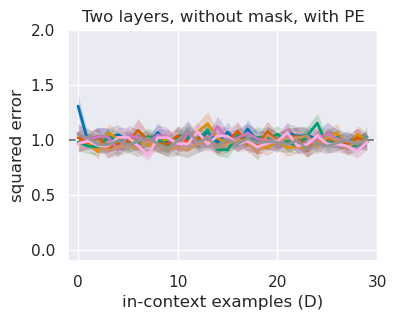}
    \includegraphics[scale=0.55]{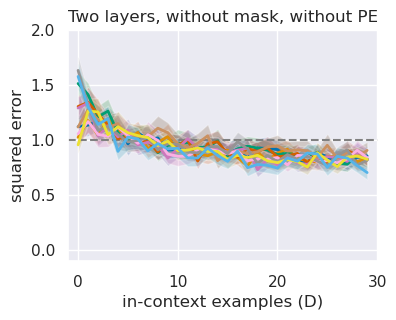}
    \includegraphics[scale=0.55]{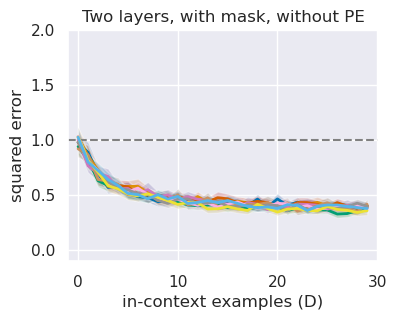}
    \includegraphics[scale=0.55]{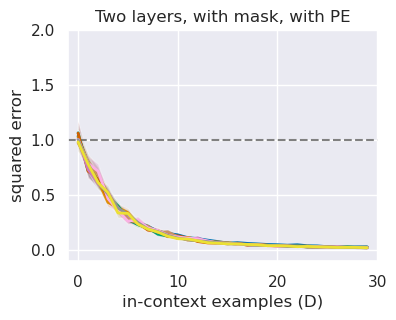}
    \caption{Ten repetitions of the simulations for one/two-layer transformers with/without PE and with/without mask.}
    \label{fig:10_repetition}
\end{figure}

\subsection{Additional Simulation Regarding PE}

In this simulation, we aim to verify the effectiveness of PE given large $p$, and examine the performance when $D$ keeps growing. We follow the implementation of \cite{garg2022can}\footnote{\url{https://github.com/dtsip/in-context-learning}} in this experiment. We take $d=5$ and $p=12$. We configure different number of examples ($D$) in the experiment. 

From Table \ref{tab:pe}, there are two observations. First, comparing with Table \ref{tab:new_pe}, when $D=30$, $p=64$ gives a better ICL performance. Second, when taking $p=12$ and further increasing $D$ to 42, the ICL performance is still worse than Table \ref{tab:new_pe}.



\begin{table}[!ht]
\centering
\begin{tabular}{ccccc}
\toprule
& \multicolumn{2}{c}{Average} & \multicolumn{2}{c}{Std}\\
$D$ &Train loss & Test loss & Train loss & Test loss\\\hline
6  & 0.5987 & 0.6803 & 0.0611 & 0.0153 \\
12 & 0.4655 & 0.4936 & 0.0497 & 0.0070 \\
18 & 0.4248 & 0.4099 & 0.0584 & 0.0497 \\
24 & 0.3940 & 0.3870 & 0.1241 & 0.0984 \\
30 & 0.3855 & 0.3913 & 0.1751 & 0.1855 \\
36 & 0.3192 & 0.3106 & 0.1314 & 0.1152 \\
42 & 0.2228 & 0.2310 & 0.0212 & 0.0068\\\bottomrule
\end{tabular}
\caption{ICL performance when increasing $D$. $d=5$, $p=12$. Benchmark: When $d=5$, $p=64$, $D=30$, the test loss is 0.1877 in Table \ref{tab:new_pe}.}\label{tab:pe}
\end{table}

\end{document}